%% file: mainArxiv.tex
\documentclass{article}
\usepackage{bm}
\usepackage{graphicx}
\usepackage{amsmath, amsthm, amssymb, bm, float}
\usepackage{color}
\usepackage{appendix}
\usepackage{amsfonts}
\usepackage{hyperref}
\hypersetup{
    colorlinks=true,
    linkcolor=python1,        
    citecolor=python3,        
    urlcolor=python1,      
}
\usepackage{enumitem}
\usepackage{booktabs}
\usepackage{algorithm}
\usepackage{algorithmic}
\usepackage{wrapfig}
\usepackage{siunitx}
\usepackage[labelfont=bf, font=small]{caption}
\newcommand{\E}{\mathbb{E}}
\newcommand{\R}{\mathbb{R}}
\newcommand{\tr}{\text{tr}}

\DeclareMathOperator{\rank}{rank}
\usepackage{fullpage}
\usepackage{cleveref}
\usepackage[style=numeric-comp, sorting=none]{biblatex}

\newtheorem{theorem}{Theorem}

\newtheorem{definition}{Definition}
\theoremstyle{remark}
\newtheorem{remark}{Remark}
\newtheorem{property}{Property}
\usepackage{array}
\usepackage[table]{xcolor}

\definecolor{python1}{HTML}{1f77b4}
\definecolor{python2}{HTML}{ff7f0e}
\definecolor{python3}{HTML}{2ca02c}
\definecolor{python4}{HTML}{d62728}
\definecolor{python5}{HTML}{9467bd}
\definecolor{python6}{HTML}{8c564b}
\definecolor{python7}{HTML}{e377c2}
\definecolor{python8}{HTML}{7f7f7f}
\definecolor{python9}{HTML}{bcbd22}
\definecolor{python10}{HTML}{17becf}
\definecolor{designcolor}{HTML}{1f77b4}

\colorlet{tablecolor1}{python1!30!white} 
\colorlet{tablecolor2}{python2!30!white} 
\colorlet{tablecolor3}{python3!30!white} 
\colorlet{tablecolor4}{python4!30!white} 
\colorlet{tablecolor5}{python5!30!white}
\colorlet{tablecolor6}{python6!30!white}
\colorlet{tablecolor7}{python7!30!white}
\colorlet{tablecolor8}{python8!30!white}
\colorlet{tablecolor9}{python9!30!white}
\colorlet{tablecolor10}{python10!30!white}

\usepackage[many]{tcolorbox}
\newtcolorbox{cframedun}{
  left skip = 0ex,
  boxrule=0.25mm,
  enhanced,
  arc=0pt,
  outer arc=0pt,
  left=8pt,
  colframe=black,
  colback= white
}

\usepackage{tikz}
\usetikzlibrary{shapes.geometric, arrows.meta, positioning}
\usetikzlibrary{calc}
\usepackage{orcidlink}

\newcommand{\independent}[1]{\perp\!\!\!\perp}

\newenvironment{keyword}
{\par\noindent\textbf{Keywords:} \itshape}
{\par}

\title{Optimal Linear Baseline Models for Scientific Machine Learning}

\author{
  Alexander DeLise\textsuperscript{\orcidlink{0009-0002-9383-5682}}\thanks{Department of Scientific Computing and Department of Mathematics,  Florida State University, Tallahassee, FL. \texttt{ard22l@fsu.edu}} 
  \and
  Kyle Loh\textsuperscript{\orcidlink{0009-0000-6625-3770}}\thanks{Department of Mathematics and Statistics,  Florida Atlantic University, Boca Raton, FL. \texttt{kloh2023@fau.edu}} 
  \and
  Krish Patel\textsuperscript{\orcidlink{0009-0008-5636-8652}}\thanks{Department of Mathematics, Emory University, Atlanta, GA. \texttt{$\{$krish.patel, matthias.chung$\}$@emory.edu }} 
  \and 
  Meredith Teague\textsuperscript{\orcidlink{0009-0008-6559-3982}}\thanks{Department of Quantitative Theory and Methods, Emory University, Atlanta, GA. \texttt{meredith.teague@emory.edu}} 
  \and
  Andrea Arnold\textsuperscript{\orcidlink{0000-0003-3003-882X}}\thanks{Department of Mathematical Sciences, Worcester Polytechnic Institute, Worcester, MA. \texttt{anarnold@wpi.edu}} 
  \and
  Matthias Chung\textsuperscript{\orcidlink{0000-0001-7822-4539}}\footnotemark[3]
}

\begin{document}

\maketitle

\begin{abstract}
    Across scientific domains, a fundamental challenge is to characterize and compute the mappings from underlying physical processes to observed signals and measurements. While nonlinear neural networks have achieved considerable success, they remain theoretically opaque, which hinders adoption in contexts where interpretability is paramount. In contrast, linear neural networks serve as a simple yet effective foundation  for gaining insight into these complex relationships. In this work, we develop a unified theoretical framework for analyzing linear encoder-decoder architectures through the lens of Bayes risk minimization for solving data-driven scientific machine learning problems. We derive closed-form, rank-constrained linear and affine linear optimal mappings for forward modeling and inverse recovery tasks. Our results generalize existing formulations by accommodating rank-deficiencies in data, forward operators, and measurement processes. We validate our theoretical results by conducting numerical experiments on datasets from simple biomedical imaging, financial factor analysis, and simulations involving nonlinear fluid dynamics via the shallow water equations. This work provides a robust baseline for understanding and benchmarking learned neural network models for scientific machine learning problems.
\end{abstract}


\section{Introduction}\label{intro}
In nearly every scientific discipline, a central challenge lies in modeling, computing, and understanding the functional relationships between signals, measurements, and their underlying physical processes. These mappings typically manifest in three fundamental forms: \emph{forward modeling}, \emph{inference}, and \emph{autoencoding}. While mathematical models often provide insight into these relationships, they are frequently inadequate for real-world prediction and analysis due to limitations in analytical tractability, computational feasibility, or algorithmic robustness.

The advent of scientific machine learning (ML) has led to a paradigm shift, where data-driven methods, particularly neural networks, have emerged as powerful tools for learning complex input-output relations directly from data. Unlike traditional model based approaches, neural networks are capable of overcoming longstanding issues such as computational complexity and scalability issues, model misspecification, and the ill-posedness inherent to many scientific problems \cite{goodfellow2016deep}.

A central strength of neural networks is their capacity to project inputs into a lower-dimensional latent space before mapping to targets, a principle commonly realized in autoencoder and encoder-decoder architectures. Neural networks compress inputs into latent representations that emphasize essential structure, removing redundancy and mitigating noise. This compression-based representation is not only central to dimensionality reduction and denoising models but also instrumental in forward and inference tasks, where the underlying mapping often resides in a lower-dimensional subspace than the observed data.

Despite the success of neural networks, they remain tools whose inner workings, particularly mathematical foundations, still await illumination. This perception serves as a major barrier to the widespread adoption of neural networks in rigorous scientific applications, where interpretability, stability, and theoretical guarantees are especially paramount. While nonlinear neural networks dominate modern ML, linear networks offer a uniquely tractable setting in which meaningful properties of the learned systems, such as generalization, robustness and approximation behavior, can be understood. With this, linear networks play a pivotal role in uncovering theoretical principles that extend to nonlinear models. 

Although a growing body of research has explored optimal linear neural networks, the literature remains fragmented and incomplete. Most existing works address only isolated components of the theory, with insights often confined to either the ML or linear algebra communities, limiting meaningful integration across disciplines. While significant attention has been devoted to dimensionality reduction techniques such as principal component analysis (PCA), their role within the broader framework of linear neural networks, particularly beyond autoencoding settings, remains underdeveloped. Moreover, only a small number of studies approach linear networks from a Bayes risk perspective, leaving critical issues such as ill-posedness and data redundancy largely unexamined within a rigorous theoretical framework. Consequently, the field still lacks a unified, comprehensive understanding of optimal linear networks--an absence that impedes the translation of foundational insights to more complex, nonlinear architectures. 

\paragraph{Contributions.} In this work, we aim to bolster the theoretical foundations for linear neural networks by drawing on fundamental concepts from linear algebra and statistical decision theory. More precisely, the key contributions of this work are as follows.
\begin{itemize}
    \item By leveraging a Bayes risk-based approach, we introduce a unified theoretical framework that yields optimal low-rank linear solutions in forward modeling and inverse mapping scenarios from a machine learning and autoencoding perspective.
    \item We extend these optimal solutions to accommodate practical scenarios, including cases where the forward operator or input data are rank-deficient or noisy.
    \item 
    We substantiate our theoretical results through comprehensive experiments using low-rank linear neural networks on both synthetic and real-world datasets. Here, our contributions are threefold. 
    \begin{itemize}
        \item  First, we numerically confirm the theoretical optimality of our approach on imaging data. 
        \item Second, we demonstrate that our method outperforms standard benchmarks in financial factor analysis, highlighting its practical relevance. 
        \item Third, we illustrate with the shallow water equations that nonlinearity in the system does not guarantee that nonlinear neural network models outperform linear ones. Linear models often capture the essential dynamics remarkably well and, without the need for hyperparameter tuning or complex learning procedures, offer a powerful and interpretable baseline. This underscores the strength of our approach in delivering accurate results while avoiding the complexity typically associated with nonlinear models.  
    \end{itemize}
\end{itemize}

Our work is structured as follows. In \Cref{background}, we provide an introduction to the scientific problems we consider and survey the foundational contributions from machine learning and linear algebra that are directly relevant to the theoretical framework developed in this work. \Cref{methods} introduces the foundational machines learning, statistical, and linear algebra tools necessary for our analysis--namely, dimensionality reduction, forward modeling, and inverse modeling.
Building on this, \Cref{theory} presents our generalized theoretical results for linear encoder-decoder architectures across a range of settings, including forward and inverse mappings and the special cases of autoencoding and data denoising. In \Cref{medmnist,finance,swe}, we demonstrate the practical relevance of our approach through numerical experiments, highlighting the effectiveness of linear neural networks in various scientific computing applications. Finally, \Cref{conclusion} summarizes our key contributions and outlines promising directions for future research.

\section{Background}\label{background}
At the heart of the scientific method lies the process of constructing, validating, and refining models that explain observed phenomena. This cycle naturally gives rise to three fundamental modeling paradigms. Forward modeling reflects the predictive step, where a known model is used to simulate outcomes from given inputs--mirroring hypothesis testing and experimentation. Conversely, inverse problems addresses the challenge of inferring unknown causes or parameters from observed data, akin to interpreting experimental results to uncover underlying mechanisms. Autoencoding, meanwhile, captures the process of abstraction and compression, distilling complex observations into essential latent representations. Together, these paradigms provide a foundational framework for understanding and formalizing the data-model relationship that underpins scientific inquiry.

These seemingly diverse mathematical problems can, at their core, be unified under a single mathematical formulation
\begin{equation}\label{eq:generalProblem}
    \mathbf{y} = \mathbf{F} (\mathbf{x}) + \bm{\varepsilon}.
\end{equation}
Here, $\mathbf{y} \in \mathcal{Y} \subseteq  \mathbb{R}^m$ is a vector of observations; $\mathbf{F}: \mathcal{X} \to \mathcal{Y}$ is a forward parameter-to-observation process; $\mathbf{x} \in \mathcal{X} \subseteq \mathbb{R}^n$ is a vector of parameters; and $\bm{\varepsilon} \in \mathbb{R}^m$ is a noise vector. This general formulation encompasses forward modeling (i.e., compute $\mathbf{y}$ given $\mathbf{x}$, $\mathbf{F}$, and potential noise $\bm{\varepsilon}$), inverse recovery (find $\mathbf{x}$ given $\mathbf{y}$ and $\mathbf{F}$) and autoencoding (given $\mathbf{x}$ and $\mathbf{y}$ find $\mathbf{F}\in \mathcal{F}$) tasks. 

A forward problem formulated as in \Cref{eq:generalProblem} seeks to propagate through the forward mapping $\mathbf{F}$ that relates parameters $\mathbf{x}$ to observations $\mathbf{y}$. Approximating the forward process is challenging when the true forward model is unavailable, computationally costly, or embedded within complex simulations \cite{frangos2010surrogate}. Prominent data-driven surrogate modeling techniques include kernel methods, physics-informed neural networks (PINNs), and general data-driven neural networks, all of which provide architectural flexibility but can be limited by dataset quality, overfitting, and extrapolation challenges \cite{santin2021kernel, raissi2019physics, cuomo2022scientific}. The efficacy of data-driven neural networks for forward modeling has recently been shown for processes like linear differential equations and control laws \cite{jo2019deep, lupu2024exact}.

Next, consider an inverse problem in the form of \Cref{eq:generalProblem}, which seeks to determine the parameters $\mathbf{x}$ given a noisy measurement $\mathbf{y}$ and knowledge of the forward process $\mathbf{F}$. Inverse problems are notoriously difficult due to potential non-invertibility of the forward operator $\mathbf{F}$, instability from noise amplification, ill-posedness, and the curse of dimensionality \cite{tarantola2005inverse, hansen2010discrete, hadamard1902problemes, bellman1961adaptive}. Practical inversions often rely on approximate or nonlinear forward models, leading to non-convex optimization problems with multiple local minima and sensitivity to initialization, thus requiring heuristic priors, regularization, advanced computational methods, and extensive labeled data, while making uncertainty quantification difficult \cite{arridge2019solving}. Established methods for solving inverse problems include variational, iterative, and Bayesian regularization \cite{benning2018modern, burger2021variational, bungert2019solution, boct2012iterative, bachmayr2009iterative, pereira2015empirical, calvetti2018inverse, arridge2019solving}, as well as a variety of deep learning strategies that use data-driven neural networks \cite{ongie2020deep, kamyab2022deep, afkhamLearningRegularizationParameters2021}. These techniques have notably advanced computational imaging, medical, and physical modeling applications \cite{jin2017deep, song2021solving, mccann2017convolutional, lucas2018using, liang2020deep}.

Motivated by inverse problems, several studies have also examined linear least squares estimators due to their ability to yield interpretable closed-form solutions. Several classical results have established minimal-norm solutions to linear inverse problems in a variety of ill-conditioned settings \cite{jackson1972interpretation, foster1961application, kay1993fundamentals, wiener1949extrapolation, izenman1975reduced}, and modern works have reinterpreted these results from a ML perspective as regularized or rank-constrained least-squares problems under a  Bayes risk framework \cite{dashti2017bayesian, idier2008bayesian, stuart2010inverse, hart2025paired, chung2024paired, chung2015optimal, spantini2015optimal}, much like we explore in our work. We also consider autoencoders--a specific type of neural network that learn compact, latent representations of data by exploiting underlying low-dimensional structures--making them valuable for tasks such as denoising, compression, surrogate modeling, and inverse problems across scientific machine learning applications \cite{goh2019solving, hart2025paired, udell2019big, berahmand2024autoencoders, gondara2016medical, masci2011stacked}.

In general, a neural network’s ability to approximate complex systems, such as forward or inverse processes, is underpinned by their universal approximation capabilities \cite{hornik1989multilayer}, though deep architectures often suffer from interpretability issues and their “black-box” nature \cite{lipton2018mythos, zhang2021survey}. While deep, nonlinear neural networks often dominate practical applications, neural network-based methods can and should be evaluated against well-understood baselines. With numerous architectural choices, hyperparameters, and training setups, meaningful comparisons across different ML approaches often become opaque. In contrast, the linear models studied in this work require little to no hyperparameter tuning and admit optimal mappings, making them ideal reference points. We provide a comprehensive analysis of such models, along with code and baselines that support fair and reproducible evaluation of more complex methods. To the best of our knowledge, this type of systematic investigation is currently lacking in the ML literature and offers new insights not yet widely explored by the community.

We present linear neural networks as rank-constrained linear encoder-decoder architectures, which serve as an interpretable, computationally efficient baseline. Exploring works related to linear neural networks, we see that linear autoencoders provide a direct link to classical low-rank matrix approximation theory. In particular, when trained to minimize squared reconstruction error, linear autoencoders recover the same subspaces as PCA \cite{baldi1989neural, bourlardAutoassociationMultilayerPerceptrons1988, plautPrincipalSubspacesPrincipal2018, baoRegularizedLinearAutoencoders2020}, reflecting an underlying rank-constrained optimization problem of which theory was first provided by \textcite{eckart1936approximation, mirsky1960symmetric, schmidt1907theorie} and later generalized by \textcite{friedland2007generalized}. This perspective has motivated extensive recent work on both analytical and numerical frameworks for low-rank matrix approximation, including settings where autoencoders are used as learnable, data-adaptive approximators \cite{halkoFindingStructureRandomness2011, boutsidisNearOptimalColumnBasedMatrix2013, cohenDimensionalityReductionKMeans2015, zhenyuezhangLowRankMatrixApproximation2013, kalooraziSubspaceOrbitRandomizedDecomposition2018, wangDimensionalityReductionStrategy2015, mounayerRankReductionAutoencoders2025, indykLearningBasedLowRankApproximations2019, babacanSparseBayesianMethods2012, mazumderLearningLowRankLatent2024}. Motivated by these connections, our work focuses on linear neural networks trained under explicit rank constraints, not only for their analytical tractability and interpretability, but also for their foundational role in bridging classical low-rank theory with modern scientific machine learning. We rely on the work of \textcite{friedland2007generalized} to derive our closed-form optimal mappings from a Bayes risk minimization perspective and reveal connections to the related problems of forward modeling and inverse recovery.

\section{Methods}\label{methods}
In this section, we present the primary method used in our work: linear encoder-decoder architectures optimized via Bayes risk minimization. We also provide relevant notation and mathematical preliminaries to contextualize our theory.

\subsection{Encoder-Decoder Architectures}\label{sec:methods_ED}
A popular framework for solving scientific ML problems is the encoder-decoder architecture. In this setup, the encoder maps a variable $\mathbf{x} \in \mathbb{R}^n$ to a low-dimensional latent representation $\mathbf{z} \in \mathbb{R}^r$ typically with $r \leq  n$. The decoder then attempts to map $\mathbf{z}$ to another vector $\mathbf{y} \in \mathbb{R}^ m$. One can define a general encoder-decoder network via parametric mappings; however, since the focus of our work is to use linear and affine linear encoder-decoder architectures, we define only those.

Let $\mathbf{E} \in \mathbb{R}^{r \times n}$ denote a trainable encoder matrix that maps an input $\mathbf{x}$ to a latent variable $\mathbf{z} = \mathbf{E}\mathbf{x}$. Let $\mathbf{D} \in \mathbb{R}^{m \times r}$ be a trainable decoder matrix, accompanied by a bias vector $\mathbf{b} \in \mathbb{R}^m$, such that the output is given by the affine mapping $\mathbf{y} = \mathbf{D}\mathbf{z} + \mathbf{b}$. This defines an affine linear encoder-decoder architecture, where a single bias vector in the decoder is sufficient. In the linear case, this bias vector $\mathbf{b}$ vanishes, i.e., $\mathbf{b}=\mathbf{0}$. While affine linear neural network layers are widely used in ML, in our setting they arise naturally as straightforward extensions of purely linear mappings. In what follows, we focus on linear mappings, with affine linear variants discussed subsequently. The entire encoder-decoder network can be composed into a single affine linear or linear mapping using $\mathbf{A}\in \mathbb{R}^{m\times n}$, where 
\[
\mathbf{y} = \mathbf{A}\mathbf{x}+\mathbf{b} = \mathbf{DE}\mathbf{x}+\mathbf{b}.
\]
In the case where we model an inverse process via an encoder-decoder, we may swap the roles of $\mathbf{x}$ and $\mathbf{y}$. \Cref{fig:edFig} depicts a typical encoder-decoder architecture for forward modeling. Assuming that $r\leq \min\{m,n\}$, we see that the $\operatorname{rank}(\mathbf{A})\leq r$, and thus the dimension of the latent space is reflected in the linear architecture as a rank constraint on the mapping $\mathbf{A}$. In practice, one must choose a bottleneck dimension $r$ and then impose the constraint $\operatorname{rank}(\mathbf{A})\leq r$ when learning the mapping $\mathbf{A}$. 

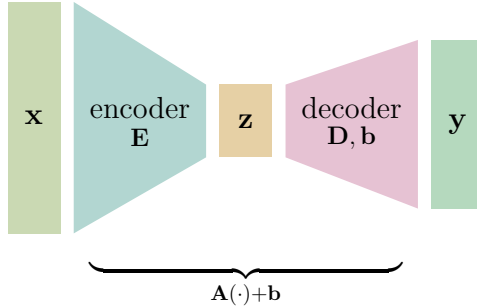
\begin{figure}[ht]
  \centering
  \scalebox{0.8}{\input{tikzFigs/edFig.tikz}}  
  \caption{%
    End-to-end encoder-decoder architecture. The input
    $\mathbf{x}$ is mapped by $\mathbf{Ex}$ onto
    a latent variable $\mathbf{z}$, then decoded by
    $\mathbf{y} \approx \mathbf{Dz} + \mathbf{b}$ to obtain a target vector $\mathbf{y}$.
    }
  \label{fig:edFig}
\end{figure}

The remainder of this paper will focus on such linear (and correspondingly affine linear) encoder-decoders and autoencoders, and their application to scientific ML problem formulations. We note that in the literature, the terms encoder-decoder and autoencoder are often used interchangeably. We too will use them interchangeably, with the precise architecture readily inferable depending on the problem context.

\subsection{A Bayes Risk Approach}\label{sec:methods_bayes}

Bayes risk minimization seeks to minimize the expected loss between an estimator of the unknown parameters in the problem at hand, given the observed data distribution and a specified loss function \cite{carlin1997bayes}. Empirical Bayes risk minimization extends this idea by calculating corresponding empirical estimates, if the underlying data distribution is unknown. Low rank matrix approximation through the Bayes risk minimization lens remains a popular approach for data scientists, with applications in many fields including medical imaging, waveform inversion, electromagnetic imaging, and computer vision \cite{hart2025paired, chung2013computing, chung2024paired, chung2017optimal, chung2014efficient, chung2015optimal, kingma2013auto, goh2019solving}. 

It is common to model the unknown parameters $X$ and the corresponding observations $Y$ as jointly distributed according to some probability distribution $P(X, Y)$. As realizations of the corresponding random variables, $\mathbf{x}$ and $\mathbf{y}$ are generated according to model specified in \Cref{eq:generalProblem}. In such settings, it is natural to take a statistical decision-theoretic view, where the goal is to design estimators or mappings that minimize the expected loss of our encoder-decoder (or autoencoder) under the underlying data distribution. Bayes risk minimization attempts to quantify the cost incurred by a decision rule over the joint distribution between the observed data and ground-truth parameters.

We formalize this notion in the context of our work by the following. Consider a single-layer linear encoder-decoder network, $\mathbf{A} =\mathbf{DE}$, as defined in \Cref{sec:methods_ED}, where an input $\mathbf{x} \in \mathbb{R}^n$ is mapped to an output $\mathbf{y} \in \mathbb{R}^m$ with an intermediate latent space representation $\mathbf{z} \in \mathbb{R}^r$. 

In forward modeling situations, our objective is to reconstruct $Y$ from our observations of $X$, and thus we aim to find a rank-constrained linear map $\mathbf{A}$ that minimizes the expected squared reconstruction error
\begin{equation}\label{eq:classicForwardProb}
    \underset{\mathrm{rank}(\mathbf{A}) \leq r}{\min} \ \E \ \left\| \mathbf{A}X - Y \right\|_2^2,
\end{equation}
where $\|\,\cdot\,\|_2$ denotes the $\ell_2$-norm. For inverse problems, we do not observe $X$ directly but have access to measurements $Y$ and seek to reconstruct $X$, in which case we aim to find a linear map $\mathbf{A}$ through the same optimization problem as in \Cref{eq:classicForwardProb} with the roles of $X$ and $Y$ switched. As discussed above, this formulation straightforwardly extends to the affine linear case.

This objective function is precisely the Bayes-optimal formulation for squared loss in linear settings. The linear encoder-decoder thus serves as a data-driven way to approximate Bayes-optimal estimators while enforcing structural constraints such as low rank. We rely on this formulation to derive the optimal rank-constrained linear mappings in \Cref{theory}.

\subsection{Notation and Preliminaries}
We now provide the relevant notation, definitions, and mathematical preliminaries which we use in our derivation of optimal encoder-decoders, and begin with the truncated singular value decomposition.

\begin{definition}[Rank-$r$ Truncated Singular Value Decomposition]\label{def:rTSVD}
    Let $\mathbf{W} \in \R^{m \times n}$ be a matrix with singular value decomposition (SVD)
    \[
        \mathbf{W}
        = \mathbf{U}_{\mathbf{W}} \bm{\Sigma}_{\mathbf{W}} \mathbf{V}_{\mathbf{W}}^\top,
    \]
    where $\mathbf{U}_{\mathbf{W}} = [\mathbf{u}_1, \mathbf{u}_2, \dots, \mathbf{u}_m] \in \mathbb{R}^{m \times m}$ is an orthogonal matrix whose columns $\mathbf{u}_i$ are the left singular vectors, $\bm{\Sigma}_{\mathbf{W}} \in \mathbb{R}^{m \times n}$ is a diagonal matrix with singular values $\sigma_1 \geq \sigma_2 \geq \cdots \geq \sigma_n \geq 0$ on its diagonal, and $\mathbf{V}_{\mathbf{W}} = [\mathbf{v}_1, \mathbf{v}_2, \dots, \mathbf{v}_n] \in \mathbb{R}^{n \times n}$ is an orthogonal matrix whose columns $\mathbf{v}_j$ are the right singular vectors. Define the rank-$r$ truncated SVD  (TSVD) of $\mathbf{W}$, for positive integer $r  \leq \min\{m, n\}$, as
    \[
        \mathbf{W}_r = \mathbf{U}_{\mathbf{W},r} \bm{\Sigma}_{\mathbf{W},r} \mathbf{V}_{\mathbf{W},r}^\top =
        \begin{bmatrix}
            \vert & \vert & & \vert \\
            \mathbf{u}_1 & \mathbf{u}_2 & \cdots & \mathbf{u}_r \\
            \vert & \vert & & \vert
        \end{bmatrix}
        \begin{bmatrix}
            \sigma_1 & & \\
            & \sigma_2 & \\
            & & \ddots & \\
            & & & \sigma_r
        \end{bmatrix}
        \begin{bmatrix}
            \mathbf{v}_1^\top \\
            \mathbf{v}_2^\top \\
            \vdots \\
            \mathbf{v}_r^\top
        \end{bmatrix},
    \]
    where the subscript $r$ denotes that only the $r$ greatest singular values and their corresponding left and right singular vectors are retained. The rank-$r$ truncation of a product of matrices is denoted as 
    \[
        \left( \mathbf{ABC} \right)_r.
    \]
\end{definition}

\begin{property}\label{property:r_TSVD}
Note that $\mathbf{W}$ has only $k=\operatorname{rank}(\mathbf{W})$ nonzero singular values. Therefore if $r \geq k$, the zero singular values do not contribute anything to the construction and $\mathbf{W}_r = \mathbf{W}_{k} = \mathbf{W}$. This also implies that for $r\geq k$, $\mathbf{W}_r =\mathbf{U}_{\mathbf{W},k}
\mathbf{\Sigma}_{\mathbf{W},k}
\mathbf{V}_{\mathbf{W},k}$.
\end{property}

In many of our expressions, especially those involving generalized inverses and low-rank surrogates for operator learning, the Moore-Penrose pseudoinverse plays a central role. Recall its construction via the SVD:
\begin{definition}[Moore-Penrose Pseudoinverse \cite{moore1920reciprocal,penrose1955generalized}]\label{def:pseudoinverse}
    Let $\mathbf{W} \in \mathbb{R}^{m \times n}$ be a matrix with SVD $\mathbf{W} = \mathbf{U}_{\mathbf{W}} \bm{\Sigma}_{\mathbf{W}} \mathbf{V}_{\mathbf{W}}^\top$. The Moore--Penrose pseudoinverse of $\mathbf{W}$ is defined as
    \[
        \mathbf{W}^\dagger
        = \mathbf{V}_{\mathbf{W}} \bm{\Sigma}_{\mathbf{W}}^\dagger \mathbf{U}_{\mathbf{W}}^\top,
    \]
    where $\bm{\Sigma}_{\mathbf{W}}^\dagger \in \mathbb{R}^{n \times m}$ is obtained by taking the reciprocal of each nonzero singular value of $\mathbf{W}$ and transposing the resulting diagonal matrix, i.e.,
    \[
        \bm{\Sigma}_{\mathbf{W}}^\dagger
        = \begin{bmatrix}
            \sigma_1^{-1} & & \\
            & \sigma_2^{-1} & \\
            & & \ddots \\
            & & & \sigma_k^{-1} \\
            & & & & \mathbf{0}_{(n - k) \times (m - k)}
        \end{bmatrix},
    \]
    with $k = \operatorname{rank}(\mathbf{W})$ and the zero block sized to complete $\bm{\Sigma}_{\mathbf{W}}^\dagger$ to dimensions $n \times m$. This unique matrix satisfies the four Penrose conditions
    \[
        \mathbf{W} \mathbf{W}^\dagger \mathbf{W} = \mathbf{W}, 
        \quad 
        \mathbf{W}^\dagger \mathbf{W} \mathbf{W}^\dagger = \mathbf{W}^\dagger, 
        \quad 
        (\mathbf{W} \mathbf{W}^\dagger)^\top = \mathbf{W} \mathbf{W}^\dagger, 
        \quad \text{and} \quad 
        (\mathbf{W}^\dagger \mathbf{W})^\top = \mathbf{W}^\dagger \mathbf{W}.
    \]
\end{definition}

The final definition establishes the relevant notation for the orthogonal projections that isolate the column and row space components of a matrix, which helps in formulating and solving the rank-constrained approximation problem for this work.

\begin{definition}[Left and Right Projection Matrices]\label{def:projection}
    Let $\mathbf{W} \in \mathbb{R}^{m \times n}$ have SVD $\mathbf{W} = \mathbf{U}_{\mathbf{W}} \bm{\Sigma}_{\mathbf{W}} \mathbf{V}_{\mathbf{W}}^\top$. Define the left and right projection matrix of $\mathbf{W}$ as $ \mathbf{P}_{\mathbf{W}}^{\mathrm{L}}$ and $\mathbf{P}_{\mathbf{W}}^{\mathrm{R}}$. These are then given by 
    \[
        \mathbf{P}_{\mathbf{W}}^{\mathrm{L}}
        = \mathbf{W} \mathbf{W}^\dagger
        = \mathbf{U}_{\mathbf{W}} \mathbf{U}_{\mathbf{W}}^\top
        = \sum_{i=1}^{\operatorname{rank}(\mathbf{W})} \mathbf{u}_i \mathbf{u}_i^\top 
        \quad \text{and} \quad 
         \mathbf{P}_{\mathbf{W}}^{\mathrm{R}} 
        = \mathbf{W}^\dagger \mathbf{W}
        = \mathbf{V}_{\mathbf{W}} \mathbf{V}_{\mathbf{W}}^\top
        = \sum_{i=1}^{\operatorname{rank}(\mathbf{W})} \mathbf{v}_i \mathbf{v}_i^\top,
    \]
    where $\mathbf{u}_i$ and $\mathbf{v}_i$ are the left and right singular vectors of $\mathbf{W}$. These projection matrices orthogonally project onto the column space and row space of $\mathbf{W}$, respectively.
\end{definition}

We also briefly recall some important properties of the trace and Frobenius norm operators: 
\begin{itemize}
    \item The trace operator, denoted $ \operatorname{tr}(\cdot) $, satisfies the cyclic property $ \operatorname{tr}(\mathbf{A} \mathbf{B}) = \operatorname{tr}(\mathbf{B} \mathbf{A}) $ whenever the product is defined, and it commutes with expectation, i.e., $ \mathbb{E}[\operatorname{tr}(X)] = \operatorname{tr}(\mathbb{E}[X])$ for any integrable random matrix $X$.
    \item The Frobenius norm of a matrix $ \mathbf{A} \in \mathbb{R}^{m \times n}$  is unitarily invariant, i.e., $ \left\| \mathbf{UAV} \right\|_{\mathrm{F}} = \left\| \mathbf{A} \right\|_{\mathrm{F}} $ for any orthogonal matrices $\mathbf{U}$ and $\mathbf{V}$. It is also quadratic in the sense that $ \left\| \mathbf{A} \right\|_{\mathrm{F}}^2 = \operatorname{tr}(\mathbf{A}^\top \mathbf{A}) $, and it aligns naturally with the Hilbert-Schmidt inner product through the identity $\langle \mathbf{A}, \mathbf{B} \rangle = \operatorname{tr}(\mathbf{A}^\top \mathbf{B})$.
\end{itemize}

With these foundational elements in place, we proceed to the principal result of \textcite{friedland2007generalized}, which yields a closed-form solution to the rank-constrained matrix minimization problem involving composed operators--crucial for the derivations that follow.

\begin{theorem}[Generalized Rank-Constrained Matrix Approximation {\cite{friedland2007generalized}}]\label{th:FriedlandAndTorokhti}
    Let matrices $\mathbf{A} \in \mathbb{C}^{m \times n}$, $\mathbf{B} \in \mathbb{C}^{m \times p}$, and $\mathbf{C} \in \mathbb{C}^{q \times n}$ be given. Then
    \[
        \mathbf{W} = \mathbf{B}^\dagger \left( \mathbf{P}_{\mathbf{B}}^{\mathrm{L}} \mathbf{A} \mathbf{P}_{\mathbf{C}}^{\mathrm{R}} \right)_r \mathbf{C}^\dagger
    \]
    is a solution to the minimization problem
    \[
        \min_{\operatorname{rank}(\mathbf{W}) \leq r} \ \left\| \mathbf{A} - \mathbf{BWC} \right\|_\mathrm{F},
    \]
    having the minimal $\left\| \mathbf{W} \right\|_\mathrm{F}$. This solution is unique if and only if either
    \[
        r \geq \operatorname{rank} \left( \mathbf{P}_{\mathbf{B}}^{\mathrm{L}} \mathbf{A} \mathbf{P}_{\mathbf{C}}^{\mathrm{R}} \right) 
    \]
    or
    \[
        1 \leq r < \operatorname{rank} \left( \mathbf{P}_{\mathbf{B}}^{\mathrm{L}} \mathbf{A} \mathbf{P}_{\mathbf{C}}^{\mathrm{R}} \right)
        \quad \text{and} \quad
        \sigma_r \left( \mathbf{P}_{\mathbf{B}}^{\mathrm{L}} \mathbf{A} \mathbf{P}_{\mathbf{C}}^{\mathrm{R}} \right)
        > \sigma_{r+1} \left( \mathbf{P}_{\mathbf{B}}^{\mathrm{L}} \mathbf{A} \mathbf{P}_{\mathbf{C}}^{\mathrm{R}} \right).
    \]
\end{theorem}

\begin{proof}
    See \cite{friedland2007generalized}.
\end{proof}

\section{Theoretical Results}\label{theory}
In this section, we establish the theoretical framework for our analysis of linear encoder-decoder networks. Our objective is to derive optimal estimates for these foundational architectures, focusing on key structural properties and analytical tools required to rigorously study their behavior.

\subsection{Forward End-to-End Problem}
\paragraph{General Problem Formulation.}
In forward surrogate modeling, one aims to approximate a noisy forward transformation, represented by Equation~\eqref{eq:generalProblem}, using a low-rank model $\mathbf{A}$. This optimal solution captures how to best emulate the effect of $\mathbf{F}$ while respecting the statistical structure of $X$ and the presence of noise. We denote the statistical independence between two random variables as, e.g., $\mathcal{E} \independent{} X$.

\begin{theorem}[Forward End-to-End Problem]\label{th:e2eFwd}
    Let $\mathbf{F} \in \mathbb{R}^{m \times n}$ be a linear forward operator,  $X$ be a random variable with finite first moment and with second moment $\mathbf{\Gamma}_{X} \in \mathbb{R}^{n \times n}$ and symmetric factorization $\mathbf{\Gamma}_{X} = \mathbf{L}_{X} \mathbf{L}_{X}^\top$, and let $\mathcal{E}$ be unbiased random noise with second moment $\mathbf{\Gamma}_{\mathcal{E}} \in \mathbb{R}^{m \times m}$. Assuming $\mathcal{E}  \perp\!\!\!\perp X$, then for positive integer $r$, 
    \begin{equation} 
        \mathbf{\widehat{A}} = (\mathbf{FL}_X)_r\mathbf{L}_X^\dagger\label{eq:solForward}
    \end{equation}
     is an optimal solution to the forward end-to-end minimization problem
    \begin{equation}      
    \underset{\mathrm{rank}(\mathbf{A}) \leq r}{\min} \hspace{0.15cm}
        \E \left\| \mathbf{A} X - (\mathbf{F} X + \mathcal{E}) \right\|_2^2,
        \label{eq:objFcnForward}
    \end{equation}
having minimal Frobenius norm $\| \mathbf{A} \|_\mathrm{F}$. This solution is unique if and only if $r\geq \operatorname{rank}(\mathbf{FL}_X)$, or $1 \leq r < \operatorname{rank}(\mathbf{FL}_X)$ and $\sigma_r(\mathbf{FL}_X) > \sigma_{r+1}(\mathbf{FL}_X)$. 
\end{theorem}

Before presenting the proof of the theorem, we highlight several important special cases that depend on the ranks of the operators involved. These arise from the structure of the data-related operator $\mathbf{L}_X$ and the forward operator $\mathbf{F}$. To formalize the case distinctions, let us define $p = \min\{m, n\}$ as the minimal dimension of the output and input spaces. We denote the rank of the data-related matrix $\mathbf{L}_X$ by $k$, with $k \leq n$, and the rank of the forward operator $\mathbf{F}$ by $\ell$, where $\ell \leq p$. These rank parameters will guide the characterization of different regimes in the learned solution.

Depending on the imposed rank $r$ of the learned mapping relative to the rank of the composed operator $\mathbf{F} \mathbf{L}_X$, we may derive several simplified and interpretable forms for the optimal solution $\widehat{\mathbf{A}}$. In particular, when $r \geq \operatorname{rank}(\mathbf{F} \mathbf{L}_X)$, the simplifications in \Cref{tab:casesforward} hold. These distinct cases are separately handled in the proof of \Cref{th:e2eFwd}, and key takeaways include:
\begin{itemize}
    \item Each simplification applies when $r \geq \operatorname{rank}(\mathbf{F} \mathbf{L}_X)$, ensuring the learned map is sufficiently expressive.
    \item When the data matrix $\mathbf{L}_X$ is full-rank ($k = n$) and $r \geq \ell$, the learned operator, as expected, exactly recovers the linear forward operator, i.e., $\widehat{\mathbf{A}} = \mathbf{F}$.
    \item When $\mathbf{L}_X$ is rank-deficient, whether due to redundant data or because the second moment matrix is estimated from a limited number of samples, the optimal solution corresponds to a projection of the forward operator onto the column space of $\mathbf{L}_X$, i.e.,
    $\widehat{\mathbf{A}} = \mathbf{F} \mathbf{U}_{\mathbf{L}_X,k} \mathbf{U}_{\mathbf{L}_X,k}^\top.$
    This projection reflects the limitation that information outside the span of $\mathbf{L}_X$ cannot be recovered.
\end{itemize}

\renewcommand{\arraystretch}{1.5}
\begin{table}
\caption{Different cases of the optimal forward end-to-end map $\widehat{\mathbf{A}}$.}
\centering
\begin{minipage}{0.5\textwidth}
\centering\small
\begin{tabular}{|
    >{\columncolor{tablecolor1}}c|
    >{\columncolor{tablecolor2}}c|
    >{\columncolor{tablecolor3}}c|
}
  \hline
  & $\ell = p$ & $\ell < p$ \\
  \hline
  $k = n$ &
    \begin{tabular}{@{}c@{}} $r \geq n$ \\[2pt] $\widehat{\mathbf{A}} = \mathbf{F}$ \end{tabular} &
    \begin{tabular}{@{}c@{}} $r \geq \ell$ \\[2pt] $\widehat{\mathbf{A}} = \mathbf{F}$ \end{tabular} \\    
  \hline
  $k < n$ &
    \begin{tabular}{@{}c@{}} $r \geq \min\{k,\ell\}$ \\[2pt] $\widehat{\mathbf{A}} = \mathbf{F} \mathbf{U}_{\mathbf{L}_X,k} \mathbf{U}_{\mathbf{L}_X,k}^\top$ \end{tabular} &
    \begin{tabular}{@{}c@{}} $r \geq \min\{k,\ell\}$ \\[2pt] $\widehat{\mathbf{A}} = \mathbf{F} \mathbf{U}_{\mathbf{L}_X,k} \mathbf{U}_{\mathbf{L}_X,k}^\top$ \end{tabular} \\
  \hline
\end{tabular}
\end{minipage}
\hspace{1em}
\begin{minipage}{0.2\textwidth}
\centering
\begin{tabular}{|>{\columncolor{tablecolor1}}c|}
  \hline
  \rowcolor{tablecolor1}
  $p = \min\{m, n\}$ \\
  \hline
  \rowcolor{tablecolor2}
   $\ell = \rank(\mathbf{F})$ \\
  \hline
  \rowcolor{tablecolor3}
  $k = \rank(\mathbf{L}_X)$ \\
  \hline
\end{tabular}
\end{minipage}

\vspace{0.75em}

\label{tab:casesforward}
\end{table}

\begin{proof}\label{prf:forward}
    Let $Y = \mathbf{F} X + \mathcal{E}$. Then 
    \begin{align}
        \E[Y Y^\top] &= \mathbf{\Gamma}_Y = \E[(\mathbf{F}X + \mathcal{E})(\mathbf{F}X + \mathcal{E})^\top] = \mathbf{F} \mathbf{\Gamma}_X \mathbf{F}^\top + \mathbf{\Gamma}_\mathcal{E} \label{eq:EYY^T}, \\
        \E[Y X^\top] &= \E[ (\mathbf{F} X + \mathcal{E}) X^\top ] = \E[\mathbf{F} X X^\top] + \E[ \mathcal{E} X^\top] = \mathbf{F} \mathbf{\Gamma}_X \label{eq:EYX^T}, \quad \text{and} \\
        \E[X Y^\top] &= \E[ X(\mathbf{F} X + \mathcal{E})^\top ] = \E[ X X^\top \mathbf{F}^\top ] + \E[ X \mathcal{E}^\top ] = \mathbf{\Gamma}_X \mathbf{F}^\top \label{eq:EXY^T}.
    \end{align}
    Now the objective function in \Cref{eq:objFcnForward} can be expressed as 
    \[
       \begin{aligned}
        \E \left\| \mathbf{A} X - Y \right\|_2^2
        &= \E \left[ \tr\left( \left( \mathbf{A} X - Y \right) \left( \mathbf{A} X - Y \right)^\top \right) \right] \\
        &= \tr\left( \mathbf{A}^\top \mathbf{A} \, \E \left[ X X^\top \right] \right) 
        - \tr\left( \mathbf{A}^\top \, \E \left[ Y X^\top \right] \right) 
        - \tr\left( \mathbf{A} \, \E \left[ X Y^\top \right] \right) 
        + \tr\left( \E \left[ Y Y^\top \right] \right) \\
        &= \tr\left( \mathbf{A}^\top \mathbf{A} \, \mathbf{\Gamma}_X \right) 
        - \tr\left( \mathbf{A}^\top \mathbf{F} \mathbf{\Gamma}_X \right) 
        - \tr\left( \mathbf{A} \mathbf{\Gamma}_X \mathbf{F}^\top \right) 
        + \tr\left( \mathbf{\Gamma}_Y \right),
    \end{aligned}
    \]
    and due to properties of the trace, we get
    \[
        \E \left\| \mathbf{A} X - Y \right\|_2^2
        = \tr\left(\mathbf{A}^\top \mathbf{A} \mathbf{\Gamma}_X\right)
        - 2\tr\left(\mathbf{A}^\top \mathbf{F} \mathbf{\Gamma}_X\right)
        + \tr\left(\mathbf{\Gamma}_Y\right).
    \]
    Invoking the identity between the trace and squared Frobenius norm gives us
    \[
        \E \left\| \mathbf{A} X - Y \right\|_2^2
        = \left\| \mathbf{A} \mathbf{L}_X \right\|_\mathrm{F}^2
        - 2\tr\left(\mathbf{A}^\top \mathbf{F} \mathbf{\Gamma}_X\right)
        + \left\| \mathbf{L}_Y \right\|_\mathrm{F}^2.
    \]
    Further, by completing the square, we can verify that
    \begin{equation}\label{eq:quadext}
        \E \left\| \mathbf{A} X - Y \right\|_2^2  
        = \left\| \mathbf{A} \mathbf{L}_X - \mathbf{F} \mathbf{L}_X \right\|_\mathrm{F}^2
        - \left\| \mathbf{F} \mathbf{L}_X \right\|_\mathrm{F}^2
        + \left\| \mathbf{L}_Y \right\|_\mathrm{F}^2.
    \end{equation}
    Since we aim to minimize $\E \left\| \mathbf{A} X - Y \right\|_2^2$ with respect to $\mathbf{A}$, the last two terms in \Cref{eq:quadext} do not contribute to this optimization problem, and we get
    \begin{equation}\label{eq:optForward}
        \min_{\operatorname{rank}(\mathbf{A}) \leq r}
        \ \left\| \mathbf{A} \mathbf{L}_X - \mathbf{F} \mathbf{L}_X \right\|_\mathrm{F}^2.
    \end{equation}
    We refer to \Cref{th:FriedlandAndTorokhti} for the general solution of \Cref{eq:optForward} as
    \begin{align*}
        \widehat{\mathbf{A}} 
        &= \left( \mathbf{F} \mathbf{L}_X \mathbf{P}_{\mathbf{L}_X}^\text{R} \right)_r \mathbf{L}_X^\dagger 
        = \left( \mathbf{F} \mathbf{L}_X \mathbf{V}_{\mathbf{L}_X,k} \mathbf{V}_{\mathbf{L}_X,k}^\top \right)_r \mathbf{L}_X^\dagger \\
        &= \left( \mathbf{F} \mathbf{U}_{\mathbf{L}_X,k} \mathbf{\Sigma}_{\mathbf{L}_X,k} \mathbf{V}_{\mathbf{L}_X,k}^\top \mathbf{V}_{\mathbf{L}_X,k} \mathbf{V}_{\mathbf{L}_X,k}^\top \right)_r \mathbf{L}_X^\dagger
        = \left( \mathbf{F} \mathbf{L}_X \right)_r \mathbf{L}_X^\dagger, 
    \end{align*}
    where $\mathbf{P}^\text{R}$ and $\mathbf{P}^\text{L}$ refer to the left and right projection matrices as defined in \Cref{def:projection}.

    Now we consider specific cases where the general solution can be further simplified. For each of the cases, we only consider $r \geq \operatorname{rank}(\mathbf{FL}_X)$, which implies that $(\mathbf{FL}_X)_r = \mathbf{FL}_X$ according to \Cref{property:r_TSVD}.
    \begin{enumerate}
        \item Full-rank $\mathbf{L}_X$ and $\mathbf{F}$, i.e., $k = n$ and $\ell = p$. Note that $\operatorname{rank}(\mathbf{F} \mathbf{L}_X) = p$ and $\mathbf{L}_X^\dagger = \mathbf{L}_X^{-1}$. Therefore,
        \[
            \widehat{\mathbf{A}} = \left( \mathbf{F} \mathbf{L}_X \right)_r \mathbf{L}_X^\dagger = 
            \mathbf{F} \mathbf{L}_X \mathbf{L}_X^{-1} = \mathbf{F}.
        \]
    
        \item Full-rank $\mathbf{L}_X$ and rank deficient $\mathbf{F}$, i.e., $k = n$ and $\ell < p$. Note that $\operatorname{rank}(\mathbf{F} \mathbf{L}_X) = \ell$ and 
        $\mathbf{L}_X$ is invertible. For $r \geq \ell$, we have
        \[
            \widehat{\mathbf{A}} = \left( \mathbf{F} \mathbf{L}_X \right)_r \mathbf{L}_X^\dagger 
            = \mathbf{F} \mathbf{L}_X \mathbf{L}_X^{-1} = \mathbf{F}.
        \] 
    
        \item Rank deficient $\mathbf{L}_X$ and full-rank $\mathbf{F}$, i.e., $k < n$ and $\ell = p$. Note that $\operatorname{rank}(\mathbf{F} \mathbf{L}_X) = b \leq \min\{k, p\}$ and according to \Cref{def:projection},
        \begin{align*}
            \mathbf{L}_X \mathbf{L}_X^\dagger =\mathbf{U}_{\mathbf{L}_X,k}\mathbf{U}_{\mathbf{L}_X,k}^\top.
        \end{align*}
        Thus for $r \geq b$ we thus have
        \begin{align*}
            \widehat{\mathbf{A}} 
            &= \left( \mathbf{F} \mathbf{L}_X \right)_r \mathbf{L}_X^\dagger 
            = \mathbf{F} \mathbf{L}_X \mathbf{L}_X^\dagger \\
            &= \mathbf{F} \mathbf{U}_{\mathbf{L}_X,k} \mathbf{U}_{\mathbf{L}_X,k}^\top.
        \end{align*}
    
        \item Rank deficient $\mathbf{L}_X$ and $\mathbf{F}$, i.e., $k < n$ and $\ell < p$. Note that $\operatorname{rank}(\mathbf{F} \mathbf{L}_X) = b \leq \min\{k, \ell\}$. Again $\mathbf{L}_X \mathbf{L}_X^\dagger = \mathbf{U}_{\mathbf{L}_X,k} \mathbf{U}_{\mathbf{L}_X,k}^\top$. Thus for $r \geq b$ we have
        \[
        \begin{aligned}
            \widehat{\mathbf{A}} 
            &= \left( \mathbf{F} \mathbf{L}_X \right)_r \mathbf{L}_X^\dagger 
            = \mathbf{F} \mathbf{L}_X \mathbf{L}_X^\dagger = \mathbf{F} \mathbf{U}_{\mathbf{L}_X,k} \mathbf{U}_{\mathbf{L}_X,k}^\top.
        \end{aligned}
        \]
    \end{enumerate}
\end{proof}

We now present the affine linear version of the forward end-to-end problem, which includes a bias term to offset the usage of covariance matrices instead of second moments. We observe that the result takes the same form as its purely linear counterpart.

\begin{theorem}[Affine Linear Forward End-to-End Problem]\label{th:e2eFwdAffLin}
    Let $\mathbf{F} \in \mathbb{R}^{m \times n}$ be a linear forward operator,  $X$ be a random variable with finite mean $\E [X] = \bm{\mu}_X \in \R^n$ and covariance $\mathbf{S}_{X} \in \mathbb{R}^{n \times n}$ with symmetric factorization $\mathbf{S}_{X} = \mathbf{K}_{X} \mathbf{K}_{X}^\top$, and let $\mathcal{E}$ be unbiased random noise with second moment $\mathbf{\Gamma}_{\mathcal{E}} \in \mathbb{R}^{m \times m}$. Assuming $\mathcal{E} \perp\!\!\!\perp X$, then for positive integer $r$, 
    \[ 
        \mathbf{\widehat{A}} = (\mathbf{FK}_X)_r\mathbf{K}_X^\dagger \quad \text{and} \quad 
        \widehat{\mathbf{b}}=(\mathbf{F} - \widehat{\mathbf{A}})\bm{\mu}_X
    \]
    is an optimal solution to the affine linear forward end-to-end minimization problem
    \begin{equation}      
        \underset{\mathrm{rank}(\mathbf{A}) \leq r}{\min} \hspace{0.15cm}
        \E \left\| \mathbf{A} X + \mathbf{b} - (\mathbf{F} X + \mathcal{E}) \right\|_2^2,
        \label{eq:objFcnForwardAffLin}
    \end{equation}
    having minimal Frobenius norm $\| \mathbf{A} \|_\mathrm{F}$. This solution is unique if and only if either $r\geq \operatorname{rank}(\mathbf{FK}_X)$, or $1 \leq r < \operatorname{rank}(\mathbf{FK}_X)$ and $\sigma_r(\mathbf{FK}_X) > \sigma_{r+1}(\mathbf{FK}_X)$. 
\end{theorem}

\begin{proof}\label{prf:e2eForAffLin}
    Let $Y = \mathbf{F} X + \mathcal{E}$, and recall \Cref{eq:EYY^T,eq:EYX^T,eq:EXY^T} from \Cref{th:e2eFwd}. Also note that 
    \[
         \E[Y] = \E[\mathbf{F}X + \mathcal{E}] = \mathbf{F} \E[X] + \E[\mathcal{E}] = \mathbf{F} \boldsymbol{\mu}_X = \boldsymbol{\mu}_Y.
    \]
    Note that $\mathbf{F} \boldsymbol{\mu}_X=\boldsymbol{\mu}_Y$. Now the objective function in \Cref{eq:objFcnForwardAffLin} can be expressed as 
    \begin{align}
        \E \ \left\| \mathbf{A} X + \mathbf{b} - Y \right \|_2^2 
        &= \E \left\| \mathbf{A} \left(X - \bm{\mu}_X\right) - \left(Y - \bm{\mu}_Y\right) + \left(\mathbf{b} - \left(\bm{\mu}_Y - \mathbf{A} \bm{\mu}_X \right) \right) \right\|_2^2\notag \\
        &= \E \left\| \mathbf{A} \left(X - \bm{\mu}_X\right) - \left(Y - \bm{\mu}_Y\right) \right\|_2^2 + \left\| \mathbf{b} - \left(\bm{\mu}_Y - \mathbf{A} \bm{\mu}_X \right) \right\|_2^2 \label{eq:transformedObjAffLinForE2E}
    \end{align}
    since the cross terms vanish under the expectation operator, and there is no expectation on the second component since it is deterministic. The second component is minimized by setting
    \begin{equation}\label{eq:optbAffLinForE2E}
        \widehat{\mathbf{b}} = (\mathbf{F} - \mathbf{\widehat{A}})\bm{\mu}_X.
    \end{equation}
    Now substitute $\bar{X} = X - \bm{\mu}_X$ and $\bar{Y} = Y - \bm{\mu}_Y$, which transforms the first component of \Cref{eq:transformedObjAffLinForE2E} to
    \[
        \E \left\| \mathbf{A} \left(X - \bm{\mu}_X\right) - \left(Y - \bm{\mu}_Y\right) \right\|_2^2 = \E \left\| \mathbf{A} \bar{X} - \bar{Y} \right\|_2^2,
    \]
    which is exactly in the form of \Cref{th:e2eFwd} and admits the minimizer
    \[
        \widehat{\mathbf{A}} = \left( \mathbf{F} \mathbf{K}_X\right)_r \mathbf{K}_X^\dagger 
        \quad \text{and} \quad 
        \widehat{\mathbf{b}} = (\mathbf{F} - \widehat{\mathbf{A}})\bm{\mu}_X.
    \]
    One may follow a similar process in the proof of \Cref{th:e2eFwd} to determine the affine linear version of the special cases for the forward end-to-end problem outlined in \Cref{tab:casesforward}.
\end{proof}

\paragraph{Autoencoding, a Special Case.} 
We may consider autoencoding a special case in of the forward end-to-end problem in which our forward operator is the identity and the data $X$ are assumed to be directly accessible and uncorrupted by noise. This motivates the study of self-reconstruction, where the goal is to approximate the identity mapping using a linear operator $\mathbf{A}$ under a prescribed rank constraint. Specifically, we seek the best low-rank linear transformation $\mathbf{A}$ such that $\mathbf{A}X \approx X$ in expectation. This formulation captures applications like dimensionality reduction and data compression, where the goal is to preserve the most informative structures within a reduced representation. Our result here relates directly to PCA, whose derivation can be found in many works, including, e.g., \textcite{baldi1989neural, bourlardAutoassociationMultilayerPerceptrons1988}; however, we arrive to our conclusion via Bayes risk minimization from a ML perspective, as that found in \cite{chung2025good, hart2025paired, chung2024paired}.

\begin{remark}[Autoencoding]\label{rem:autoencoding}
    In the autoencoding problem, we assume that the forward operator $\mathbf{F} = \mathbf{I}_n$ and that our data is not corrupted by any noise. We do not require the random variable $X$ to have finite first moment. Our optimization problem becomes
    \[
        \underset{\mathrm{rank}\left(\mathbf{A}\right) \leq r}{\min} \; \E \left\| \mathbf{A}X - X \right\|_2^2,
    \]
    and we may simplify the general result of \Cref{th:e2eFwd} to yield
    \begin{align*}
        \widehat{\mathbf{A}} = (\mathbf{FL}_X)_r\mathbf{L}_X^\dagger &= (\mathbf{L}_X)_r\mathbf{L}_X^\dagger 
        = \mathbf{U}_{\mathbf{L}_X,r} \mathbf{U}_{\mathbf{L}_X,r}^\top.
    \end{align*}
    This solution is unique if and only if  $r \geq \operatorname{rank}(\mathbf{L}_X)$, or $1 \leq r < \operatorname{rank}(\mathbf{L}_X)$ and $\sigma_r\left(\mathbf{L}_X\right) > \sigma_{r+1}\left(\mathbf{L}_X\right)$. We also note two interesting properties of the optimal linear autoencoder, as well as its affine linear counterpart:
    \begin{itemize}
        \item \textbf{Identity Recovery.} In the case where $r\geq \operatorname{rank}(\mathbf{L}_X)$ and $\mathbf{L}_X$ is full-rank, i.e., $\operatorname{rank}(\mathbf{L}_X) = n$, the optimizer is, as expected, the identity since
        \[
            \mathbf{\widehat{A}} = \mathbf{U}_{\mathbf{L}_X} \mathbf{U}_{\mathbf{L}_X}^\top = \mathbf{I}_n.
        \]

        \item \textbf{Non-Unique Factorization.} While the optimal autoencoding mapping $\widehat{\mathbf{A}} = \mathbf{U}_{\mathbf{L}_X,r} \mathbf{U}_{\mathbf{L}_X,r}^\top$ is unique (under the conditions stated above), its decomposition into encoder and decoder matrices is not. That is, for any invertible matrix $\mathbf{Q} \in \mathbb{R}^{r \times r}$, we may write
        \[
            \widehat{\mathbf{A}} = \left(\mathbf{U}_{\mathbf{L}_X,r} \mathbf{Q}\right)\left(\mathbf{Q}^{-1} \mathbf{U}_{\mathbf{L}_X,r}^\top\right),
        \]
        showing that the encoder-decoder pair is identifiable only up to an invertible transformation in the latent space.

        \item \textbf{Affine Linear Autoencoder.} In the affine linear case, our optimization problem becomes
        \[
            \underset{\substack{\mathrm{rank}(\mathbf{A}) \leq r \\ \mathbf{b} \in \mathbb{R}^n}}{\min} \ \E \ \| \mathbf{A} X + \mathbf{b} - X \|_2^2,
        \]
        in which the optimal noiseless autoencoder mapping follows naturally from \Cref{th:e2eFwdAffLin} and with a similar simplification process as before to yield 
        \[
            \widehat{\mathbf{A}} = 
            \mathbf{U}_{\mathbf{K}_X,r}\mathbf{U}_{\mathbf{K}_X,r}^\top
            \quad \text{and} \quad
            \widehat{\mathbf{b}} = (\mathbf{I}_n - \widehat{\mathbf{A}} ) \bm{\mu}_X.
        \] 
        This solution is unique if and only if  $r \geq \operatorname{rank}(\mathbf{K}_X)$, or $1 \leq r < \operatorname{rank}(\mathbf{K}_X)$ and $\sigma_r\left(\mathbf{K}_X\right) > \sigma_{r+1}\left(\mathbf{K}_X\right)$.
    \end{itemize}
\end{remark}

\subsection{Inverse End-To-End Problem}
\paragraph{General Problem Formulation.}
In inverse problems, the objective is to reconstruct the underlying signal $X$ from a noisy forward observation $Y = \mathbf{F}X + \mathcal{E}$, where $\mathbf{F}$ is a linear operator and $\mathcal{E}$ is unbiased noise independent of $X$. This recovery task motivates learning a rank-constrained operator $\mathbf{A}$ such that $\mathbf{A}Y \approx X$. 

\begin{theorem}[Inverse End-to-End Problem]\label{th:e2eInv}
    Let $\mathbf{F} \in \mathbb{R}^{m \times n}$ be a linear forward operator,  $X$ be a random variable with finite first moment and with second moment $\mathbf{\Gamma}_{X} \in \mathbb{R}^{n \times n}$, and let $\mathcal{E}$ be unbiased random noise with second moment $\mathbf{\Gamma}_{\mathcal{E}} \in \mathbb{R}^{m \times m}$. Define a new random variable $Y$ such that $Y = \mathbf{F} X + \mathcal{E}$ with second moment $\mathbf{\Gamma}_Y \in \mathbb{R}^{m \times m}$ and symmetric factorization $\mathbf{\Gamma}_{Y} = \mathbf{L}_{Y} \mathbf{L}_{Y}^\top$. Assuming $\mathcal{E}  \perp\!\!\!\perp X$, then for positive integer $r$, 
    \begin{equation}
        \widehat{\mathbf{A}} = \left(  \mathbf{\Gamma}_X \mathbf{F}^\top \mathbf{L}_Y^{\dagger \, \top} \right)_r \mathbf{L}_Y^{\dagger}\label{eq:solInverse}
    \end{equation}
    is an optimal solution to the inverse end-to-end minimization problem
    \begin{equation}\label{eq:inverse}
        \underset{\operatorname{rank}\left( \mathbf{A} \right) \leq r}{\min} \hspace{0.15cm}
        \mathbb{E} \left\| \mathbf{A} \left( \mathbf{F} X + \mathcal{E} \right) - X \right\|_2^2
        = \hspace{0.15cm}
        \mathbb{E} \left\| \mathbf{A} Y - X \right\|_2^2,
    \end{equation}
    having minimal Frobenius norm $\left\| \mathbf{A} \right\|_\mathrm{F}$. This minimizer is unique if and only if either $r\geq \operatorname{rank}(\mathbf{\Gamma}_X \mathbf{F}^\top \mathbf{L}_Y^{\dagger \, \top})$, or $1 \le r < \operatorname{rank}(\mathbf{\Gamma}_X \mathbf{F}^\top \mathbf{L}_Y^{\dagger \, \top})$ and $\sigma_r(\mathbf{\Gamma}_X \mathbf{F}^\top \mathbf{L}_Y^{\dagger \, \top}) > \sigma_{r+1}(\mathbf{\Gamma}_X \mathbf{F}^\top \mathbf{L}_Y^{\dagger \, \top})$. 
\end{theorem}

\begin{proof}
    Let $Y = \mathbf{F} X + \mathcal{E}$ as stated in the theorem, and recall \Cref{eq:EYY^T,eq:EYX^T,eq:EXY^T} in the proof of \Cref{th:e2eFwd}. The optimization problem in \Cref{eq:inverse} can be restated as 
    \[
    \begin{aligned}
        \underset{\mathrm{rank}(\mathbf{A}) \leq r}{\min} \hspace{0.15cm}
        \E \left\| \mathbf{A} Y - X \right\|_2^2
        &= \E \left[ \tr\left( \left( \mathbf{A} Y - X \right) \left( \mathbf{A} Y - X \right)^\top \right) \right] \\
        &= \tr\left( \mathbf{A} \E \left[ Y Y^\top \right] \mathbf{A}^\top \right) - \tr\left( \mathbf{A} \E \left[ Y X^\top \right] \right) - \tr\left( \E \left[ X Y^\top \right] \mathbf{A}^\top \right) + \tr\left( \E \left[ X X^\top \right] \right) \\
        &= \tr\left( \mathbf{\Gamma}_Y \mathbf{A}^\top \mathbf{A} \right) - 2 \tr\left( \mathbf{A} \mathbf{F} \mathbf{\Gamma}_X \right) + \tr\left( \mathbf{\Gamma}_X \right).
    \end{aligned}
    \]
    By using the identity between the trace and the squared Frobenius norm, the above expression becomes
    \[
        \E \left\| \mathbf{A} Y - X \right\|_2^2 
        = \left\| \mathbf{A} \mathbf{L}_Y \right\|_\mathrm{F}^2 - 2 \tr\left( \mathbf{A} \mathbf{F} \mathbf{\Gamma}_X \right) + \left\| \mathbf{L}_X \right\|_\mathrm{F}^2,
    \]
    where $\mathbf{L}_X$ is the symmetric factor of $\mathbf{\Gamma}_X$. Further, by completing the square, we can verify that
    \begin{equation}\label{eq:completedSquares}
      \E \left\| \mathbf{A} Y - X \right\|_2^2 
        = \left\| \mathbf{A} \mathbf{L}_Y \right\|_\mathrm{F}^2 - 2 \tr\left( \mathbf{A} \mathbf{F} \mathbf{\Gamma}_X \right) = \left\| \mathbf{A} \mathbf{L}_Y - \mathbf{C} \right\|_\mathrm{F}^2 - \left\| \mathbf{C} \right\|_\mathrm{F}^2 + \left\| \mathbf{L}_X \right\|_\mathrm{F}^2, 
    \end{equation}
    where $\mathbf{C} = \mathbf{\Gamma}_X \mathbf{F}^\top \mathbf{L}_Y^{\dagger \, \top}$. 
    
    To justify these substitutions for completing the square, recall from \Cref{eq:EYY^T} that $\mathbf{\Gamma}_Y = \E[YY^\top] = \mathbf{F} \mathbf{\Gamma}_X \mathbf{F}^\top + \mathbf{\Gamma}_{\mathcal{E}}$. We have that $\mathbf{F} \mathbf{\Gamma}_X \mathbf{F}^\top$ is symmetric positive semidefinite (SPSD) because it is of the form $\mathbf{GWG}^\top$ with $\mathbf{W} = \mathbf{\Gamma}_X \succeq 0$ being a second moment matrix of a real-valued random variable. $\mathbf{\Gamma}_\mathcal{E}$ is also SPSD by definition, and thus $\mathbf{\Gamma}_Y$, the sum of two SPSD matrices, is also SPSD.
    Now let $\operatorname{rank}(\mathbf{\Gamma}_Y)=p\leq m$, and consider the rank-$p$ truncated eigendecomposition 
    \[
        \mathbf{\Gamma}_Y = \mathbf{Q}_{Y,p} \, \mathbf{\Lambda}_{Y,p} \, \mathbf{Q}_{Y,p}^\top 
        = \left( \mathbf{Q}_{Y,p} \, \mathbf{\Lambda}_{Y,p}^{1/2} \right) \left( \mathbf{Q}_{Y,p} \, \mathbf{\Lambda}_{Y,p}^{1/2} \right)^\top 
        = \mathbf{L}_Y \mathbf{L}_Y^\top,
    \]
    which guarantees that $\mathbf{L}_{Y} \in \mathbb{R}^{m\times p}$ has full column rank. Thus according to \Cref{def:pseudoinverse},  $\mathbf{L}_Y^\dagger\mathbf{L}_{Y} = \mathbf{I}_p$. 
    
    With this, let $\operatorname{rank}(\mathbf{C})=q$ and note that  $\mathbf{C}_q = \mathbf{C}$, as defined in \Cref{def:rTSVD}. The last two terms in \Cref{eq:completedSquares} remain constant with respect to $\mathbf{A}$ yielding the optimization problem
    \begin{equation}\label{eq:invE2E_obj}
        \underset{\mathrm{rank}(\mathbf{A}) \leq r}{\min} 
        \ \left\| \mathbf{A} \mathbf{L}_Y - \mathbf{C} \right\|_\mathrm{F}^2.
    \end{equation}
    Utilizing \Cref{th:FriedlandAndTorokhti} and the projection matrices $\mathbf{P}^\text{R}$ and $\mathbf{P}^\text{L}$ defined in \Cref{def:projection}, we see that the rank-constrained minimization problem in \Cref{eq:invE2E_obj} has the optimal solution
    \begin{align*}
        \widehat{\mathbf{A}} 
        &= \left( \mathbf{C} \mathbf{P}_{\mathbf{C}}^\text{R} \right)_r \mathbf{L}_Y^{\dagger} 
        = \left( \mathbf{C}_{q} \mathbf{P}_{\mathbf{C}}^\text{R} \right)_r \mathbf{L}_Y^{\dagger} 
        = \left( \mathbf{U}_{\mathbf{C}, q}  \bm{\Sigma}_{\mathbf{C}, q}  \mathbf{V}_{\mathbf{C}, q}^\top \mathbf{V}_{\mathbf{C}, q} \mathbf{V}_{\mathbf{C}, q}^\top \right)_r \mathbf{L}_Y^{\dagger} \\
        &= \left( \mathbf{U}_{\mathbf{C}, q}  \bm{\Sigma}_{\mathbf{C},q} \mathbf{V}_{\mathbf{C}, q}^\top \right)_r \mathbf{L}_Y^{\dagger}
        = \mathbf{C}_r\mathbf{L}_Y^{\dagger} 
        = \left(  \mathbf{\Gamma}_X \mathbf{F}^\top \mathbf{L}_Y^{\dagger \, \top} \right)_r \mathbf{L}_Y^{\dagger}.
    \end{align*}

    Consider the special case when $r\geq q$, then $\mathbf{C}_r=\mathbf{C}$ and the unique minimizer simplifies to 
    \[
    \widehat{\mathbf{A}} = \mathbf{C} \mathbf{L}_Y^{\dagger}
    = \mathbf{\Gamma}_X \mathbf{F}^\top \mathbf{L}_Y^{\dagger \, \top} \mathbf{L}_Y^{\dagger} 
    = \mathbf{\Gamma}_X \mathbf{F}^\top \left( \mathbf{L}_Y \mathbf{L}_Y^\top \right)^\dagger
    = \mathbf{\Gamma}_X \mathbf{F}^\top \mathbf{\Gamma}_Y^\dagger.
    \]
\end{proof}

The result of \Cref{th:e2eInv} holds in all the possible ill-conditioned cases, such as when the forward map $\mathbf{F}$, data matrix $\mathbf{L}_X$, or noise matrix $\mathbf{\Gamma}_\mathcal{E}$ are rank-deficient. In practical settings, we note that the noise covariance (or, equivalently, second moment in the case our noise is unbiased) $\mathbf{\Gamma}_{\mathcal{E}}$ is typically assumed to be symmetric positive definite (SPD) to reflect full-dimensional, non-degenerate variability and to guarantee well-posedness of estimation and learning procedures \cite{kay1993fundamentals,hastie2009elements,williams2006gaussian}. 
In this special case, $\mathbf{\Gamma}_{Y}$ is also SPD and thus one can use a more efficient decomposition to compute $\mathbf{L}_Y$, such as Cholesky. However, even in the absence of this assumption our solution is stable and unique.

We now present the affine linear version of the purely linear inverse end-to-end optimizer. Again, the bias vector offsets the usage of centered covariance matrices, though we note a connection to prior literature in the full-rank case.

\begin{theorem}[Affine Linear Inverse End-to-End Problem]\label{th:e2eInvAffLin}
    Let $\mathbf{F} \in \mathbb{R}^{m \times n}$ be a linear forward operator, $X$ a random variable with mean $\mathbb{E}[X] = \bm{\mu}_X \in \mathbb{R}^n$ and covariance $\mathbf{S}_X \in \mathbb{R}^{n \times n}$, and let $\mathcal{E}$ be unbiased random noise with covariance $\mathbf{S}_{\mathcal{E}} \in \mathbb{R}^{m \times m}$. Define a new random variable $Y$ such that $Y = \mathbf{F} X + \mathcal{E}$ with covariance $\mathbf{S}_Y \in \mathbb{R}^{m \times m}$ and symmetric factorization $\mathbf{S}_{Y} = \mathbf{K}_{Y} \mathbf{K}_{Y}^\top$. Assuming $\mathcal{E} \perp\!\!\!\perp X$, then for positive integer $r$,
    \[
        \widehat{\mathbf{A}} = \left( \mathbf{S}_X \mathbf{F}^\top \mathbf{K}_Y^{\dagger \, \top} \right)_r \mathbf{K}_Y^\dagger
        \quad \text{and} \quad
        \widehat{\mathbf{b}} = (\mathbf{I}_n - \widehat{\mathbf{A}} \mathbf{F}) \bm{\mu}_X
    \]
    is an optimal solution to the affine linear inverse end-to-end minimization problem
    \begin{equation}
        \underset{\operatorname{rank}(\mathbf{A}) \leq r}{\min} \
        \mathbb{E} \left\| \mathbf{A} Y + \mathbf{b} - X \right\|_2^2,
        \label{eq:objFcnInvAffLin}
    \end{equation}
    having minimal Frobenius norm $\| \mathbf{A} \|_\mathrm{F}$. This solution is unique if and only if either $r \geq \operatorname{rank}(\mathbf{S}_X \mathbf{F}^\top \mathbf{K}_Y^{\dagger \, \top})$, or $1 \leq r < \operatorname{rank}(\mathbf{S}_X \mathbf{F}^\top \mathbf{K}_Y^{\dagger \, \top})$ and $\sigma_r\left( \mathbf{S}_X \mathbf{F}^\top \mathbf{K}_Y^{\dagger \, \top} \right) > \sigma_{r+1} \left( \mathbf{S}_X \mathbf{F}^\top \mathbf{K}_Y^{\dagger\, \top} \right) $.
\end{theorem}

\begin{proof}\label{prf:e2eInvAffLin}
    Let $Y = \mathbf{F} X + \mathcal{E}$ and note that 
    \[
        \mathbb{E}[Y] = \mathbf{F} \bm{\mu}_X= \bm{\mu}_Y.
    \]
    The objective function in \Cref{eq:objFcnInvAffLin} can be written as
    \begin{align}
        \mathbb{E} \left\| \mathbf{A} Y + \mathbf{b} - X \right\|_2^2 
        &= \mathbb{E} \left\| \mathbf{A} (Y - \bm{\mu}_Y) - (X - \bm{\mu}_X) + \left( \mathbf{b} - (\bm{\mu}_X - \mathbf{A} \bm{\mu}_Y) \right) \right\|_2^2 \notag \\
        &= \mathbb{E} \left\| \mathbf{A} \bar{Y} - \bar{X} \right\|_2^2 
        + \left\| \mathbf{b} - (\bm{\mu}_X - \mathbf{A} \bm{\mu}_Y) \right\|_2^2,
        \label{eq:transformedObjAffLinInvE2E}
    \end{align}
    where $\bar{X} = X - \bm{\mu}_X$ and $\bar{Y} = Y - \bm{\mu}_Y$. The second term is minimized by
    \begin{equation}\label{eq:optbAffLinInvE2E}
        \widehat{\mathbf{b}} = \bm{\mu}_X - \widehat{\mathbf{A}} \bm{\mu}_Y = (\mathbf{I}_n - \widehat{\mathbf{A}} \mathbf{F}) \bm{\mu}_X.
    \end{equation}
    This leaves us the first term
    \[
        \mathbb{E} \left\| \mathbf{A} \bar{Y} - \bar{X} \right\|_2^2,
    \]
    which is exactly the form of \Cref{th:e2eInv}, whereby we obtain the minimizer
    \[
        \widehat{\mathbf{A}} = \left( \mathbf{S}_X \mathbf{F}^\top \mathbf{K}_Y^{\dagger \, \top} \right)_r \mathbf{K}_Y^{\dagger}.
    \]
\end{proof}

\begin{remark}[Full-Rank Case]\label{rmk:e2eInvAffLin_foster}
    Assume the hypotheses of \Cref{th:e2eInvAffLin}. Let $k = \operatorname{rank}(\mathbf{S}_X)$, $p=\operatorname{rank}(\mathbf{K}_Y)$ and $\ell = \operatorname{rank}(\mathbf{F})$, so that
    \[
        \operatorname{rank}\left(\mathbf{S}_X \mathbf{F}^\top \mathbf{K}_Y^{\dagger \, \top}\right) 
        = \operatorname{rank}\left(\mathbf{S}_X \mathbf{FK}_Y\right) 
        = q \leq \min\{k, \ell,p \}.
    \]
   In the special case when $r \geq q$, there is no truncation and we obtain
    \[
        \widehat{\mathbf{A}} = \left( \mathbf{S}_X \mathbf{F}^\top \mathbf{K}_Y^{\dagger \, \top} \right)_r \mathbf{K}_Y^{\dagger} 
        = \mathbf{S}_X \mathbf{F}^\top \mathbf{K}_Y^{\dagger \, \top} \mathbf{K}_Y^{\dagger} 
        = \mathbf{S}_X \mathbf{F}^\top \mathbf{S}_Y^{\dagger}.
    \]
    This recovers the full-rank optimal affine linear inverse estimator originally derived by \textcite{foster1961application}, which naturally leads to its corresponding optimal bias vector
    \[
        \widehat{\mathbf{b}} = (\mathbf{I}_n - \widehat{\mathbf{A}} \mathbf{F}) \bm{\mu}_X.
    \]
\end{remark}

\paragraph{Data Denoising, a Special Case.}
In practical scenarios, observed signals are often corrupted by additive noise without a forward process. This leads to the denoising problem, where the goal is to recover the clean input $X$ from its noisy observation $X + \mathcal{E}$. The goal of the optimal denoising operator is to recover the original signal as accurately as possible while reducing the influence of noise, subject to a rank constraint on the mapping. We highlight this special case as an inverse end-to-end problem here.

\begin{remark}[Data Denoising]\label{rem:dataDenoising}
    In the data denoising problem, we assume that the forward operator $\mathbf{F} = \mathbf{I}_n$ and that our data are corrupted by additive noise. Our optimization problem becomes
    \[
        \underset{\mathrm{rank}(\mathbf{A}) \leq r}{\min} \ 
        \E \left\| \mathbf{A}(X + \mathcal{E}) - X \right\|_2^2,
    \]
    and we may simplify the general result of \Cref{th:e2eInv} to give us the solution
    \begin{equation}
        \widehat{\mathbf{A}} = \left( \mathbf{\Gamma}_X \mathbf{L}_Y^{\dagger \, \top} \right)_r \mathbf{L}_Y^{\dagger}.\label{eq:generalDenoising}
    \end{equation}
    This minimizer is unique if and only if either $r \geq \operatorname{rank}(\mathbf{\Gamma}_X  \mathbf{L}_Y^{\dagger \, \top})$, or $1 \leq r < \operatorname{rank}(\mathbf{\Gamma}_X \mathbf{L}_Y^{\dagger \, \top})$ and $\sigma_r(\mathbf{\Gamma}_X \mathbf{L}_Y^{\dagger \, \top}) > \sigma_{r+1}(\mathbf{\Gamma}_X \mathbf{L}_Y^{\dagger \, \top})$. 

    We briefly note properties of this solution and its link to prior results in literature.
    \begin{itemize}
        \item \textbf{Wiener Filter.} Let $\operatorname{rank}\left(\mathbf{\Gamma}_X \mathbf{L}_Y^{\dagger \, \top}\right) = q$  and $r \geq q$. Then the optimal solution to the data denoising problem is given by the Wiener filter \cite{wiener1949extrapolation}
        \[
            \widehat{\mathbf{A}} 
            = \mathbf{\Gamma}_X \left( \mathbf{\Gamma}_X + \mathbf{\Gamma}_{\mathcal{E}} \right)^{\dagger}.
        \]
        The Wiener filter (or, in this case, the Wiener smoothing matrix \cite{kay1993fundamentals}) is the optimal Bayesian linear estimator for recovering a signal corrupted by additive noise with the minimum mean squared error (MSE). It balances signal fidelity against noise suppression by optimally weighting components according to their relative second moments. 
        
        To derive this result note that since  $r \geq q$, the best rank-$r$ approximation of $\mathbf{\Gamma}_X \mathbf{L}_Y^{\dagger \, \top}$ is itself. Then we have
        \[
             \widehat{\mathbf{A}} = \left( \mathbf{\Gamma}_X \mathbf{L}_Y^{\dagger \, \top} \right)_r \mathbf{L}_Y^{\dagger} 
             = \mathbf{\Gamma}_X \mathbf{L}_Y^{\dagger \, \top} \mathbf{L}_Y^{\dagger}
             = \mathbf{\Gamma}_X \mathbf{\Gamma}_Y^{\dagger} = \mathbf{\Gamma}_X \left( \mathbf{\Gamma}_X +  \mathbf{\Gamma}_\mathcal{E} \right)^{\dagger}.
        \]

        \item \textbf{Affine Linear Data Denoising Problem.} In the affine linear case, our optimization problem becomes 
        \[
             \underset{\mathrm{rank}(\mathbf{A}) \leq r}{\min}
            \ \E \left\| \mathbf{A}(X + \mathcal{E}) + \mathbf{b} - X \right\|_2^2.
        \]
        The solution to this optimization problem follows naturally from \Cref{th:e2eInvAffLin} to yield
        \[
            \widehat{\mathbf{A}} = \left( \mathbf{S}_X \mathbf{K}_Y^{\dagger \, \top} \right)_r \mathbf{K}_Y^{\dagger} 
            \quad \text{and} \quad
            \widehat{\mathbf{b}} = (\mathbf{I}_n - \widehat{\mathbf{A}})\bm{\mu}_X.
        \]
        This solution is unique if and only if either $r \geq \operatorname{rank} \left( \mathbf{S}_X \mathbf{K}_Y^{\dagger \, \top} \right)$, or $1 \leq r < \operatorname{rank}\left( \mathbf{S}_X \mathbf{K}_Y^{\dagger \, \top} \right)$ and $\sigma_r\left( \mathbf{S}_X \mathbf{K}_Y^{\dagger \, \top} \right) > \sigma_{r+1}\left( \mathbf{S}_X \mathbf{K}_Y^{\dagger \, \top} \right)$.
        Analogous to the affine linear counterparts of all the other results, we observe that the mean weighted bias vector offsets the usage of the centered covariance matrix.          

        \item  \textbf{Noiseless Autoencoder Recovery}. When $\mathcal{E} = \mathbf{0}$ we recover exactly the optimal noiseless autoencoder map, as expected, but from the inverse perspective. That is, our optimization problem becomes 
        \[
            \underset{\mathrm{rank}\left(\mathbf{A}\right) \leq r}{\min} \; \E \left\| \mathbf{A}X - X \right\|_2^2,
        \]
        which has the solution
        \[
            \widehat{\mathbf{A}} = \mathbf{U}_{\mathbf{L}_X,r} \mathbf{U}_{\mathbf{L}_X,r}^\top.
        \]
        This solution is unique if and only if  $r \geq \operatorname{rank}\left(\mathbf{L}_X\right)$, or $1 \leq r < \ell$ and $\sigma_r\left(\mathbf{L}_X\right) > \sigma_{r+1}\left(\mathbf{L}_X\right)$, and all other simplifications and properties from \Cref{rem:autoencoding} naturally hold. 
        
        To derive this from the general solution of the data denoising problem, first note that $\mathbf{\Gamma}_Y = \mathbf{\Gamma}_X$ and thus $\mathbf{L}_Y = \mathbf{L}_X$, and $\operatorname{rank}(\mathbf{\Gamma}_X) = \operatorname{rank}(\mathbf{L}_X) = q$. Now we can simplify \Cref{eq:generalDenoising} to 
        \begin{align*}
            \widehat{\mathbf{A}} = \left( \mathbf{\Gamma}_X \mathbf{L}_Y^{\dagger \, \top} \right)_r \mathbf{L}_Y^{\dagger} 
            &= \left( \mathbf{L}_X (\mathbf{L}_X^\dagger \mathbf{L}_X)^\top \right)_r \mathbf{L}_X^{\dagger} 
            = \left( \mathbf{L}_X \mathbf{V}_{\mathbf{L}_X, q}  \mathbf{V}_{\mathbf{L}_X, q}^\top \right)_r \mathbf{L}_X^\dagger 
            = \mathbf{L}_{X, r} \mathbf{L}_X^\dagger 
            = \mathbf{U}_{\mathbf{L}_X,r} \mathbf{U}_{\mathbf{L}_X,r}^\top,
        \end{align*}
        where we invoke \Cref{def:rTSVD,def:projection} for our simplifications.
    \end{itemize}
\end{remark}



\Cref{table:compiled_theo_results} summarizes our main theoretical results from this section for the forward and inverse end-to-end problems, their special cases, and affine linear counterparts. 

\begin{table}[H]
\renewcommand{\arraystretch}{1.5}
\caption{Compiled theoretical results. \label{table:compiled_theo_results}}
\centering\small

\begin{tabular}{|
    >{\columncolor{tablecolor1}}>{\centering\arraybackslash}p{2.5cm}|
    >{\columncolor{tablecolor2}}p{4.5cm}|
    >{\columncolor{tablecolor3}}p{4.5cm}|
    >{\columncolor{tablecolor4}}>{\centering\arraybackslash}p{1.7cm}|
}
\hline
\textbf{Task} & \centering\textbf{Linear Solution} & \centering\textbf{Affine Linear Solution} & \textbf{Links} \\
\hline

\textbf{Forward} & 
$\widehat{\mathbf{A}} = (\mathbf{F} \mathbf{L}_X)_r \mathbf{L}_X^\dagger$ & 
$\begin{array}{rl}
\widehat{\mathbf{A}} &= \left( \mathbf{F} \mathbf{K}_X \right)_r \mathbf{K}_X^\dagger \\
\widehat{\mathbf{b}} &= (\mathbf{F} - \widehat{\mathbf{A}})\bm{\mu}_X
\end{array}$ &
\Cref{th:e2eFwd} \\
\hline

\textbf{Inverse} & 
$\widehat{\mathbf{A}} = \left( \mathbf{\Gamma}_X \mathbf{F}^\top \mathbf{L}_Y^{\dagger \, \top} \right)_r \mathbf{L}_Y^{\dagger}$ & 
$\begin{array}{rl}
\widehat{\mathbf{A}} &= \left( \mathbf{S}_X \mathbf{F}^\top \mathbf{K}_Y^{\dagger \, \top} \right)_r \mathbf{K}_Y^{\dagger} \\
\widehat{\mathbf{b}} &= (\mathbf{I}_n - \widehat{\mathbf{A}} \mathbf{F}) \bm{\mu}_X
\end{array}$ &
\Cref{th:e2eInv} \\
\hline

\textbf{Autoencoding} & 
$\widehat{\mathbf{A}} = \mathbf{U}_{\mathbf{L}_X,r} \mathbf{U}_{\mathbf{L}_X,r}^\top$ & 
$\begin{array}{rl}
\widehat{\mathbf{A}} &= \mathbf{U}_{\mathbf{K}_X,r} \mathbf{U}_{\mathbf{K}_X,r}^\top \\
\widehat{\mathbf{b}} &= (\mathbf{I}_n - \widehat{\mathbf{A}})\bm{\mu}_X
\end{array}$ &
\Cref{rem:autoencoding} \\
\hline

\textbf{Denoising} & 
$\widehat{\mathbf{A}} = \left( \mathbf{\Gamma}_X \mathbf{L}_Y^{\dagger \, \top} \right)_r \mathbf{L}_Y^{\dagger}$ & 
$\begin{array}{rl}
\widehat{\mathbf{A}} &= \left( \mathbf{\Gamma}_X \mathbf{L}_Y^{\dagger \, \top} \right)_r \mathbf{L}_Y^{\dagger} \\
\widehat{\mathbf{b}} &= (\mathbf{I}_n - \widehat{\mathbf{A}})\bm{\mu}_X
\end{array}$ &
\Cref{rem:dataDenoising} \\
\hline

\end{tabular}
\end{table}

\subsection{Empirical Bayes Risk Minimization}\label{sec:empBayesRisk}
While our theory provides a Bayes risk interpretation by characterizing the expected squared error under the true data-generating distribution, in practice, we approximate this risk using a finite training set. Consider a set of realizations $\mathbf{x}_1, \ldots, \mathbf{x}_J \in \mathbb{R}^n$ of a random variable $X$, and define the data matrix $\mathbf{X} = [\mathbf{x}_1, \ldots, \mathbf{x}_J] \in \mathbb{R}^{n \times J}$. Let $\mathbf{Y} = [\mathbf{y}_1, \ldots, \mathbf{y}_J] \in \mathbb{R}^{m \times J}$ denote the corresponding outputs, where each $\mathbf{y}_j \in \mathbb{R}^m$ is generated from $\mathbf{x}_j$ via a forward process, such as that described in \Cref{eq:generalProblem}. We may then obtain empirical estimates for the second moment and covariance of our data with the following formulae:
\begin{equation}\label{eq:empiricalXMats}
    \mathbf{\Gamma}_{\mathbf{X}} = \tfrac{1}{J} \mathbf{X} \mathbf{X}^\top 
    \quad \text{and} \quad
    \mathbf{S}_{\mathbf{X}} = \tfrac{1}{J-1} (\mathbf{X} - \bm{\mu}_{\mathbf{X}} \mathbf{1}_J^\top)(\mathbf{X} - \bm{\mu}_{\mathbf{X}} \mathbf{1}_J^\top)^\top,
\end{equation}
where $\boldsymbol{\mu}_{\mathbf{X}} \in \mathbb{R}^{784}$ is the mean image across all samples, or our cross variances
\[
    \mathbf{\Gamma}_{\mathbf{XY}} = \tfrac1J \mathbf{XY}^\top 
    \quad \text{and} \quad
    \mathbf{\Gamma}_{\mathbf{YX}} = \tfrac1J \mathbf{YX}^\top.
\]
Note the change in subscript as we are now working with empirical data.

Another approximation can be obtained working directly with the samples in the solution formulae, e.g., \Cref{eq:solForward,eq:solInverse}. For instance, consider the inverse problem case, in which the optimal rank-constrained mapping is given by \Cref{th:e2eInv}. We may plug in our empirical estimates for the second moments, and use the symmetric factorization $\mathbf{\Gamma}_\mathbf{Y}= \tfrac1J \mathbf{YY}^\top$ to yield
\begin{align*}
    \widehat{\mathbf{A}} = \left(  \mathbf{\Gamma}_X \mathbf{F}^\top \mathbf{L}_Y^{\dagger \ \top} \right)_r \mathbf{L}_Y^{\dagger} 
    &= \left(  \E [XY^\top] \, \mathbf{L}_Y^{\dagger \ \top} \right)_r \mathbf{L}_Y^{\dagger} \\
    &\approx 
    \left( \tfrac1J \mathbf{X}  \mathbf{Y}^\top \sqrt{J}\, \mathbf{Y}^{\dagger \, \top} \right)_r \, \sqrt{J} \, \mathbf{Y}^\dagger
    = \left(\mathbf{X}  \mathbf{Y}^\top \mathbf{Y}^{\dagger \, \top} \right)_r \mathbf{Y}^\dagger 
    = \left( \mathbf{X} \mathbf{V}_\mathbf{Y} \mathbf{V}_\mathbf{Y}^\top \right)_r \mathbf{Y}^\dagger.
\end{align*}
One may also consider sample average approximation to the expectation, leading to the optimization problem
\begin{equation}\label{eq:emp_bayes_inv}
    \underset{\mathrm{rank}(\mathbf{A}) \leq r}{\min} \; \tfrac{1}{J} \left\| \mathbf{A} \mathbf{Y} - \mathbf{X} \right\|_\mathrm{F}^2,
\end{equation}
corresponding to the MSE loss used during model training. The optimal solution to the rank-constrained minimization problem in \Cref{eq:emp_bayes_inv} is given by \Cref{th:FriedlandAndTorokhti} as 
\[
    \left( \mathbf{X} \mathbf{V}_\mathbf{Y} \mathbf{V}_\mathbf{Y}^\top \right)_r \mathbf{Y}^\dagger.
\]
In either case, we see that we arrive to the same estimator of the optimal mapping.

Likewise, for the forward case, we may plug in our empirical estimates into the theoretically optimal forward end-to-end minimizer from \Cref{th:e2eFwd} to yield
\begin{align*}
    \widehat{\mathbf{A}} 
    = \left(\mathbf{FL}_X\right)_r\mathbf{L}_X^\dagger
    &= \left(\mathbf{F} \mathbf{L}_X \mathbf{L}_X^\top  \mathbf{L}_X^{\dagger \, \top} \right)_r \mathbf{L}_X^\dagger 
    = \left(\mathbf{F} \mathbf{\Gamma}_X  \mathbf{L}_X^{\dagger \, \top} \right)_r \mathbf{L}_X^\dagger
    = \left( \E [YX^\top] \,\mathbf{L}_X^{\dagger \, \top} \right)_r \mathbf{L}_X^\dagger \\
    &\approx \left( \tfrac1J \mathbf{Y} \mathbf{X}^\top \sqrt{J} \ \mathbf{X}^{\dagger \, \top} \right)_r \sqrt{J} \ \mathbf{X}^\dagger
    = \left(\mathbf{Y} \mathbf{X}^\top \mathbf{X}^{\dagger \, \top} \right)_r \ \mathbf{X}^\dagger
    = \left( \mathbf{Y} \mathbf{V}_\mathbf{X} \mathbf{V}_\mathbf{X}^\top \right)_r \mathbf{X}^\dagger,
\end{align*}
where we use the symmetric factorization $\mathbf{\Gamma}_X= \tfrac1J \mathbf{XX}^\top$, and one can verify the substitution via the TSVD. When working directly with the data, our optimization problem becomes
\begin{equation}\label{eq:emp_bayes_for}
    \underset{\mathrm{rank}(\mathbf{A}) \leq r}{\min} \; \tfrac{1}{J} \left\| \mathbf{A} \mathbf{X} - \mathbf{Y} \right\|_\mathrm{F}^2,
\end{equation}
in which \Cref{th:FriedlandAndTorokhti} provides the optimal solution as 
\[
    \widehat{\mathbf{A}} =  \left( \mathbf{Y} \mathbf{V}_\mathbf{X} \mathbf{V}_\mathbf{X}^\top \right)_r \mathbf{X}^\dagger.
\]

Previous works including \textcite{baldi1989neural} and \textcite{izenman1975reduced} have derived closed-form solutions to the least-squares problems referenced in \Cref{eq:emp_bayes_inv} and \Cref{eq:emp_bayes_for}, respectively, in the setting of linear autoencoders and affine linear mappings. Although purely data-driven formulations are widely used to model relationships between datasets, the associated theoretical results remain scattered across fields. The derivations above adopt the formulation of \textcite{friedland2007generalized} to reinterpret these problems through a machine learning perspective, offering a unified framework that connects linear algebra with scientific modeling.

The aforementioned derivations also demonstrate that the least squares solution serves only as one possible empirical estimator of the theoretically optimal mappings, but not necessarily the most effective.  Notably, the least squares approach ignores prior knowledge of the underlying forward process, which can be a significant limitation in scientific machine learning contexts where such information is often available and meaningful. By contrast, the theoretical formulation permits flexibility in constructing such estimators. One may incorporate structural knowledge of the forward operator $\mathbf{F}$ or choose a specific symmetric factorization, such as a Cholesky decomposition, as this choice can significantly affect computational efficiency and numerical stability, particularly in high-dimensional settings. Empirical Bayes risk minimization thus offers a principled and computationally tractable alternative. It is compatible with standard machine learning practices while also accommodating domain-specific knowledge when available. For these reasons, we adopt the empirical Bayes risk minimization framework throughout this and subsequent sections in our numerical experiments.

\section{Numerical Experiments}\label{numerical}
We conduct three numerical experiments to evaluate and apply the theoretical framework from \Cref{theory} across diverse applications, using Bayes risk minimization as a unifying principle. These include structured image data, financial factor analysis, and nonlinear partial differential equations.

\subsection{\texttt{MedMNIST} Study}\label{medmnist}
We begin with a basic biomedical imaging example. For this investigation, we focus primarily on the general forward and inverse end-to-end problem formulations, and defer the results and discussion for all special cases to \Cref{app:medmnistFigs}.

\paragraph{Dataset and Overview.}\label{sec:medMNIST_dataset_overview}
For this setup, we utilize \texttt{MedMNIST}, a set of standardized biomedical images \cite{medmnistv2}, to run a variety of numerical experiments. Specifically, we utilize four key datasets: \texttt{TissueMNIST}, a collection of kidney microscopy images \cite{tissuemnist}; \texttt{ChestMNIST}, a hospital-scale chest x-ray database \cite{chestmnist}; \texttt{OrganAMNIST}, for axial images of livers \cite{organmnist1} \cite{organmnist2}; and \texttt{RetinaMNIST}, for retinal images \cite{medmnistv2}. Much like those in the classical \texttt{MNIST} database, these images are standardized to have a $28 \times 28$ pixel aspect ratio and they are purposefully lightweight for ML tasks. We selected these four datasets because they span a broad range of sample sizes, ranging from small datasets with about 1,000 samples to large ones with over 200,000.

We performed the following process for each dataset separately. First, we vectorized images from the datasets, transforming each $28 \times 28$ image into a data point $\mathbf{x}_j \in \mathbb{R}^{784}$. After vectorizing each image, we concatenated them column-wise to form $\mathbf{X} = [\mathbf{x}_1 \, \cdots \, \mathbf{x}_J] \in \mathbb{R}^{n \times J}$, where $J$ is the number of built-in training samples of each dataset. We form the corresponding dataset of observations $\mathbf{Y} = [ \mathbf{y}_1 \, \cdots \, \mathbf{y}_J ] \in \mathbb{R}^{m \times J}$ where the method by which we obtain $\mathbf{y}_j$ changes depending on the scenario (e.g., forward or inverse) in \Cref{theory}. We compute the empirical second moment matrices of the data, and add a small ridge term to ensure positive definiteness. This enables the efficient computation of symmetric factors via Cholesky decomposition.

We define the forward operator as a full-rank Gaussian blurring process. Specifically, we construct a fixed operator $\mathbf{F} \in \mathbb{R}^{784 \times 784}$ that performs a spatially invariant Gaussian convolution on each vectorized input image. Here, $\mathbf{F}$ is constructed by convolving each reshaped basis vector $\mathbf{e}_j$ with a $5 \times 5$ Gaussian kernel (standard deviation $s_{\mathbf{F{}}} = 1.5$) and re-vectorizing the result to form the $j$-th column. This ensures that $\mathbf{F}$ captures the linear action of convolution with the kernel and yields a full-rank square matrix consistent with the input dimension.

Additive Gaussian white noise is sampled with zero mean and covariance $\mathbf{\Gamma}_{\mathbf{E}} = s_{\mathbf{E}}^2 \mathbf{I}_{784}$, where $s_{\mathbf{E}} = 0.05$. Each column of the noise matrix $\mathbf{E} \in \mathbb{R}^{784 \times J}$ is an independent draw $\bm{\varepsilon}_j \sim \mathcal{N}(\mathbf{0}, s_{\mathbf{E}}^2 \mathbf{I}_{784})$, added to the corresponding blurred signal $\mathbf{F} \mathbf{X}$. The noisy observations are given by $\mathbf{Y} =\mathbf{F}\mathbf{X}+\mathbf{E}$.
\Cref{fig:e2eImgEvol} illustrates this process on a sample from the \texttt{ChestMNIST} dataset, showing the effects of blurring and noise.

\begin{figure}
    \centering
    \includegraphics[width=0.67\linewidth]{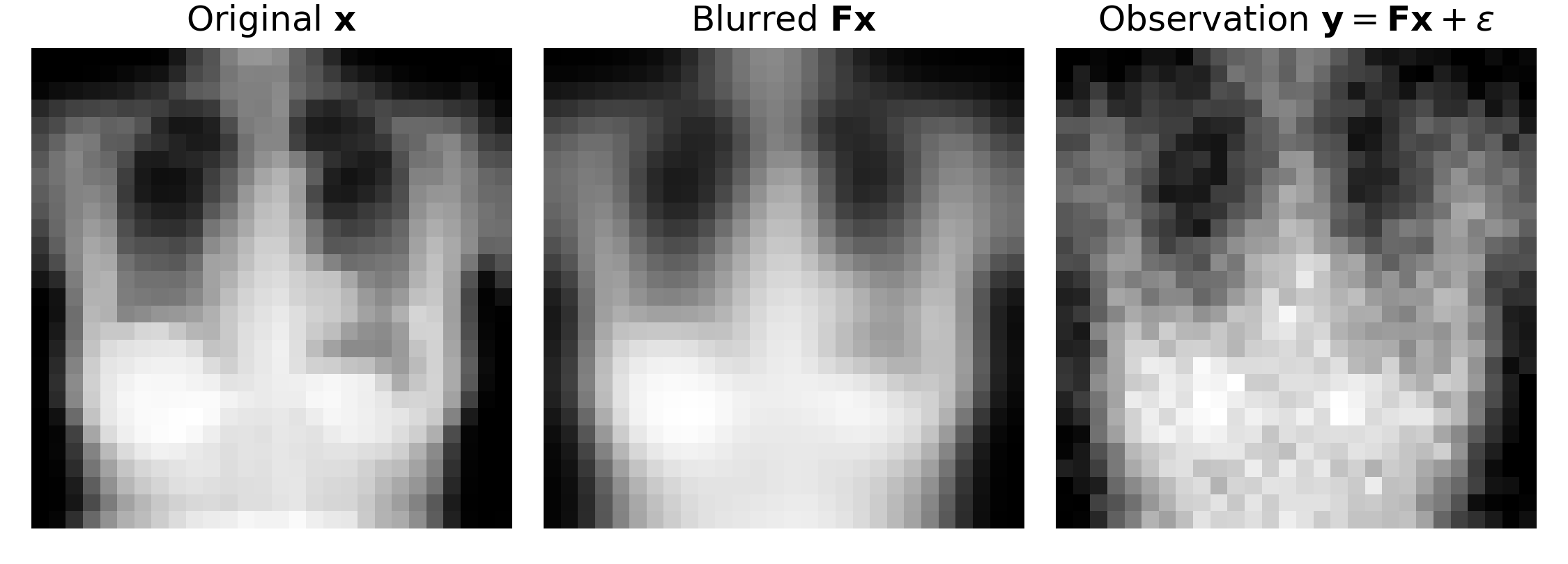}
      \caption{Example of the forward end-to-end process on sample 7181 from \texttt{ChestMNIST}. The left image shows the original input $\mathbf{x}$; the center image shows the blurred version $\mathbf{F} \mathbf{x}$ after application of the forward operator $\mathbf{F}$; and the right image shows the final observed measurement $\mathbf{y} = \mathbf{F} \mathbf{x} + \bm{\varepsilon}$ after addition of Gaussian noise $\bm{\varepsilon}$.}
    \label{fig:e2eImgEvol}
\end{figure}

\paragraph{Encoder-Decoder Architecture.} 
To learn rank-constrained mappings, we utilize the single-layer encoder-decoder architecture outlined in \Cref{sec:methods_ED}. Specifically, we implement a single-layer linear encoder-decoder architecture using \texttt{PyTorch}. The encoder $\mathbf{E}_r \in \mathbb{R}^{784 \times r}$ maps each input $\mathbf{y}$ to a latent representation $\mathbf{z} \in \mathbb{R}^r$, with $r < 784$, and the decoder $\mathbf{D}_r \in \mathbb{R}^{r \times 784}$ attempts to reconstruct the corresponding ground truth signal $\mathbf{x}$. Our encoder-decoder model optimizes the empirical Bayes risk by learning a low-rank linear operator that minimizes the average $\ell_2$ reconstruction error (MSE) between predicted and ground-truth signals. For our results, we denote the optimal and learned rank-$r$ mappings by $\widehat{\mathbf{A}}_r$ and $\mathbf{A}_r$, respectively, and keep this notation for the remainder of the section. To empirically validate the theoretical results, we construct optimal linear mappings of rank $r$ using the closed-form expressions derived in \Cref{theory} and compare them to learned mappings obtained from training the encoder-decoder architecture described above. Finally, we train our model using the \texttt{ADAM} \cite{kingma2014adam} optimizer with a learning rate of $10^{-3}$ for 200 epochs for each tested rank. 

\paragraph{Results.} We evaluate both the optimal and learned mappings across a range of target ranks $r \in \{25, 50, \dots, 775\}$. The maximum rank considered is 775 to remain within the effective rank of the data, as each image lies in $\R^{784}$ and hence the maximum possible rank is 784.

\Cref{fig:e2eFullComp} compares the average per-sample $\ell_2$ reconstruction error between optimal mappings $\widehat{\mathbf{A}}_r$ and learned encoder-decoder mappings $\mathbf{A}_r$ as a function of bottleneck rank $r$ across \texttt{MedMNIST} datasets. The results are shown separately for the forward and inverse end-to-end problems. These results empirically validate our theoretical predictions in \Cref{th:e2eFwd,th:e2eInv}, since across all datasets and for both tasks, the optimal mappings consistently outperform the learned mappings in terms of reconstruction error. 

Interestingly, we note that the reconstruction error of the forward process is slightly worse than its inverse counterpart, and we attribute this to the fact that noise serves as a regularizer in the inversion process. We also note that the performance of the optimal mappings plateaus as the rank $r$ exceeds the effective dimensionality of the data. This reflects the fact that the trailing singular values of the optimal mapping are effectively zero and thus do not contribute additional information to the reconstruction.

\Cref{fig:e2eOptVSLearned} contrasts optimal and learned reconstructions for a representative \texttt{ChestMNIST} sample at bottleneck rank $r = 25$. For both the forward and inverse cases, the optimal mappings yield reconstructions that closely match the input image, with low. Notably, the optimal forward error appears as isolated, random pixel-level noise, while the optimal inverse error reveals more structured error. In contrast, both the forward and inverse learned mappings fail to capture key details in both directions, introducing grainy artifacts and producing error maps that are both larger in intensity and more visibly patterned than those of the optimal mappings. The learned mappings fail to capture much of the underlying structure of the figure in the original image.

\begin{figure}
    \centering
    \includegraphics[width=0.49\linewidth]{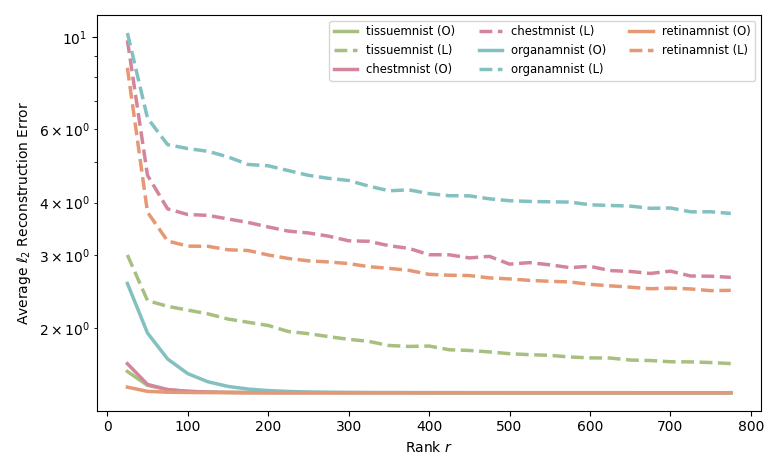}
    \includegraphics[width=0.49\linewidth]{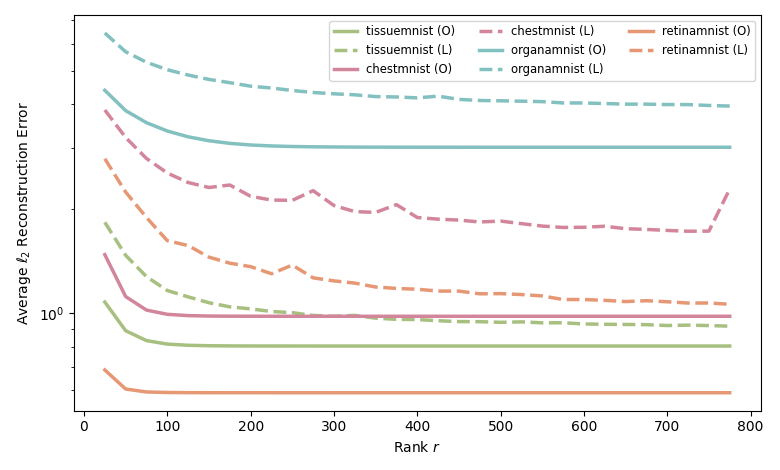}
    \caption{Average per-sample $\ell_2$ reconstruction error versus bottleneck rank $r$ across \texttt{MedMNIST} datasets for both forward and inverse end-to-end problems. \textbf{Left:} Forward process, where the goal is to learn a mapping from $\mathbf{x}$ to $\mathbf{y}$. The optimal and learned mappings minimize the empirical losses $\tfrac{1}{J} \sum_{j=1}^{J} \| \widehat{\mathbf{A}}_r^{\text{F}} \mathbf{x}_j - \mathbf{y}_j \|_2^2$ and $\tfrac{1}{J} \sum_{j=1}^{J} \| \mathbf{A}_r^{\text{F}} \mathbf{x}_j - \mathbf{y}_j \|_2^2$, respectively. \textbf{Right:} Inverse process, where the goal is to reconstruct $\mathbf{x}$ from $\mathbf{y}$. The optimal and learned mappings minimize the empirical losses $\tfrac{1}{J} \sum_{j=1}^{J} \| \widehat{\mathbf{A}}_r^{\text{I}} \mathbf{y}_j - \mathbf{x}_j \|_2^2$ and $\tfrac{1}{J} \sum_{j=1}^{J} \| \mathbf{A}_r^{\text{I}} \mathbf{y}_j - \mathbf{x}_j \|_2^2$, respectively. Solid lines denote optimal mappings $\widehat{\mathbf{A}}_r^{(\cdot)}$ (O), while dashed denote learned encoder-decoder mappings $\mathbf{A}_r^{(\cdot)}$ (L).}
    \label{fig:e2eFullComp}
\end{figure}

\begin{figure}
    \centering
    \includegraphics[width=0.85\linewidth]{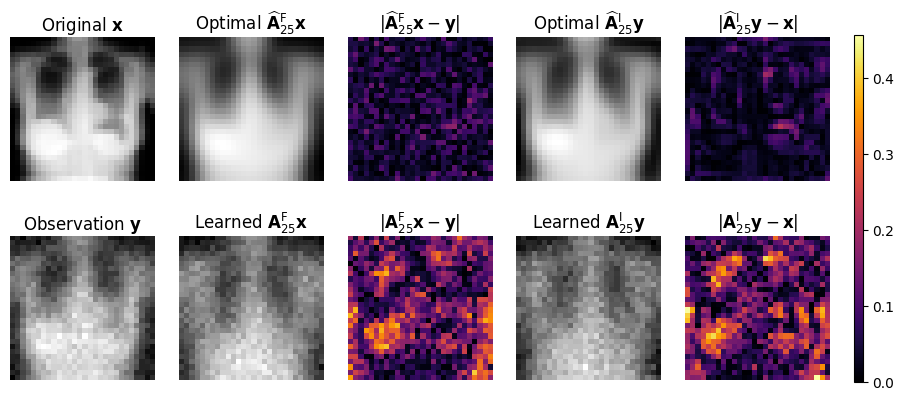}
    \caption{Example reconstruction results from \texttt{ChestMNIST} with bottleneck rank $r = 25$ for both forward and inverse end-to-end problems. In all cases, the observed measurement is denoted $\mathbf{y}$ and the ground truth is $\mathbf{x}$. \textbf{Top row:} Reconstructions using optimal mappings. Left to right: original signal $\mathbf{x}$; optimal forward reconstruction $\widehat{\mathbf{A}}_{25}^{\text{F}} \mathbf{x}$ and corresponding error map $| \widehat{\mathbf{A}}_{25}^{\text{F}} \mathbf{x} - \mathbf{y} |$; optimal inverse reconstruction $\widehat{\mathbf{A}}_{25}^{\text{I}} \mathbf{y}$ and corresponding error map $| \widehat{\mathbf{A}}_{25}^{\text{I}} \mathbf{y} - \mathbf{x}|$. \textbf{Bottom row:} Learned encoder-decoder reconstructions. Left to right: observation $\mathbf{y}$; learned forward reconstruction $\mathbf{A}_{25}^{\text{F}} \mathbf{x}$ and corresponding error map $| \mathbf{A}_{25}^{\text{F}} \mathbf{x} - \mathbf{y} |$; learned inverse reconstruction $\mathbf{A}_{25}^{\text{I}} \mathbf{y}$ and corresponding error map $| \mathbf{A}_{25}^{\text{I}} \mathbf{y} - \mathbf{x} |$.}
    \label{fig:e2eOptVSLearned}
\end{figure}

\subsection{Application to Financial Markets}\label{finance}
Financial markets are driven by a small number of risk factors that impact many assets simultaneously. Main factors such as market sentiment, sector-specific shocks, or macroeconomic tendencies are hard to quantify and observe, yet manifest as correlated price movements in assets. In the following, we investigate the theoretically derived inference of risk factors achieved by our optimal affine linear autoencoder mapping and compare its effectiveness against classical and widely used benchmark models in this setting.

\paragraph{Motivation and Overview.}
The volatile nature of the U.S. stock market makes it difficult to identify important underlying factors that drive asset pricing. Motivated by these considerations, the Fama-French Three-Factor model (FF3) \cite{fama1993common} was proposed as a widely adopted extension of the Capital Asset Pricing Model (CAPM) originally introduced in \textcite{sharpe1964capital}. The FF3 model defines the market excess return (Market), size (SMB), and value (HML) factors as three key drivers of asset pricing. Typically, these factors are disproportionate, with market excess return being the most influential and size and value more subtle. \textcite{fama1992crosssection} prove that these three factors (Market, SMB, and HML) can explain over 90\% of the variance in monthly diversified asset returns, establishing FF3 as a standard statistical model for asset pricing and portfolio modeling. For a detailed discussion on the model, including common terminology and definitions, we refer the reader to \Cref{app:financialdefs}.

PCA is another widely used technique that can extract the latent factors from asset return data by identifying linear combinations of assets that maximize variance, under the assumption that these latent factors are orthogonal \cite{jolliffe2011principal}. However, during periods of high volatility, PCA has been shown to extract components that overfit to idiosyncratic noise or clustering rather than providing economically meaningful risk signals \cite{verma2017volatility}, resulting in distorted factor estimates and poor portfolio outlooks. Nevertheless, PCA remains a popular benchmark for evaluating factor models due to its simplicity, data-driven nature, and ability to capture broad patterns in asset co-movements without relying on specific economic assumptions.

Finally, autoencoders provide a flexible framework for both linear and nonlinear dimensionality reduction that can capture complex factor structures in market data. The encoder maps high-dimensional asset returns to a low-dimensional latent space (the latent factors), while the decoder reconstructs the original returns. This architecture naturally handles the factor analysis problem: the bottleneck layer represents the latent factors, and reconstruction error measures how well the latent space explains the data. Traditional Arbitrage Pricing Theory \cite{ross1976arbitrage} suggests that the underlying market structure for assets is linear, hence we expect a linear autoencoder to identify latent factors well. For completeness, we also include a nonlinear autoencoder. 

We compare the optimal affine linear autoencoder from \Cref{rem:autoencoding} with latent dimension $r = 3$ (referred to as \textit{Optimal Affine Linear}) against several standard approaches: FF3, a PCA-based method, an untrained affine linear autoencoder with random initialization (\textit{Trained Affine Linear}), and a trained nonlinear autoencoder (\textit{Trained Nonlinear}). For further details on the model architectures, we refer the reader to \Cref{fin-arc}.
We apply the models to real-world financial datasets consisting of daily diversified asset returns, which exhibit substantially more noise than their monthly counterparts. Furthermore, we conduct an additional experiment in \Cref{synth-data-exp} by
evaluating the models on synthetic data, allowing for a direct comparison against a known ground truth. 

\paragraph{Market Data Acquisition.} \label{market-data-exp}
Our dataset contains time series data for 194 selected stocks, obtained via Python’s \texttt{yfinance} library \cite{yfinance}, which span the U.S.~equity market. All Global Industry Classification Standard (GICS) sectors are represented, capturing a wide range of market capitalizations, liquidity profiles, and volatility levels. The data covers 787 trading days, from November 11, 2021, to December 30, 2024, and contains key statistical properties, such as mean daily return and average daily volatility, that align with values reported in relevant literature \cite{wiese2020quant, campbell1997econometrics}. A full list of included stocks and sector classifications is available on our \href{https://github.com/alexdelise/CMDS-REU-MADDI}{GitHub page}.

We use the standard log-return formula to calculate daily returns, as it ensures time additivity and supports common distributional assumptions used in financial modeling. To prevent extreme price movements from skewing autoencoder training, we remove single-day changes exceeding $\pm 50\%$ from the dataset. We forward-filled missing values for up to five consecutive days to handle minor data gaps. For each asset, we split their dataset chronologically, with the first $80\%$ of observations allocated for training and the remaining $20\%$ for testing. Preserving the chronological order prevents future information from leaking into the training set, ensuring realistic and temporally consistent model evaluation.

\paragraph{Model Architecture and Performance Metrics.} 
For the Trained Affine Linear and Trained Nonlinear models, we implemented the architectures outlined in \Cref{fin-arc} with hidden layer width $H = 144$, number of assets $A = 194$, and number of time steps $T = 787$. We applied a Varimax Orthogonal Rotation \cite{kaiser1958varimax} to the latent space for interpretability by maximizing the variance of squared loadings per factor. This encourages each asset to load on a smaller number of factors while maintaining orthogonality \cite{sherin1966matrix}. However, when coupled with the random initialization and non-convexity of the optimization problem, Varimax introduces rotation-induced numerical instability, particularly when at least two factors explain similar amounts of variance \cite{browne2001overview}. This could lead to inconsistent factor orderings and orientations across different runs, resulting in high variability of factor decompositions despite consistent overall model performance (MSE). Hence, the Trained Affine Linear and Trained Nonlinear autoencoders may converge to local minima instead of the global minima. 

To address these issues, we report MSE and factor analysis metrics for the Trained Affine Linear and Trained Nonlinear models as an average over 100 runs with corresponding standard deviation. Within each run, we trained the Trained Affine Linear and Nonlinear autoencoders over 150 and 100 epochs, respectively.

\paragraph{Results.} \Cref{tab:market_mse} displays the reconstruction MSE of the four tested methods. \Cref{tab:market_varimax} displays the cumulative explained variance (CEV) of the tested methods, where CEV, a commonly utilized statistic in finance, measures the percentage of variance from the market dataset captured by the three most important latent factors of each model. Higher CEVs indicate that a model's factors provide a better representation of the data. We note that the CEVs of all our models were relatively low (below $40\%$ overall) because we use a daily asset returns dataset, which is extremely volatile and noisy, intensifying the difficulty of extracting meaningful latent structure.

As a financial benchmark, we compare FF3 factors from the Kenneth French Data Library to our latent spaces. The Kenneth R. French Data Library, see \cite{famafrenchlib} for details, records FF3 factors using a large subset of the entire stock universe, while our model-generated latent spaces are trained on a much smaller subset of the stock universe. Generating FF3 factors for our smaller stock database requires using financial metrics and resources that are often not available for public use. This discrepancy explains the relatively low CEV of the FF3 model compared to some of the other models. However, FF3 factors still allow for a general financial baseline for comparison. 

\begin{table}[H]
    \caption{MSE results for market data.}
    \label{tab:market_mse}
    \centering\small
    
    \begin{tabular}{|
    >{\columncolor{tablecolor1}}>{\centering\arraybackslash}p{3.5cm}|
    >{\columncolor{tablecolor2}}>{\centering\arraybackslash}p{5.2cm}|
}
\hline
\textbf{Method} & \textbf{Mean Squared Error} ($\times10^{-4}$) \\
\hline
Optimal Affine Linear & 
$2.88 $ \\
\hline
Trained Affine Linear & 
$4.00 \pm 0.05 $ \\
\hline
Trained Nonlinear & 
$4.06 \pm 0.02$ \\
\hline
PCA & 
$2.96 $ \\
\hline
\end{tabular}
\end{table}

\begin{table}[H]
\renewcommand{\arraystretch}{1.5}
\caption{Rotated CEV and factor balance for market data. Items in bold are the best in their respective category.
    \textbf{*}Latent factor names are assigned arbitrarily, as alignment with the FF3 structure is not guaranteed. For benchmarking, we use daily FF3 factors from the Kenneth R. French Data Library, which are computed over a broader stock universe and serve only as an approximate reference.
}
\label{tab:market_varimax}
\centering\small

\begin{tabular}{|
    >{\columncolor{tablecolor1}}>{\centering\arraybackslash}p{2.5cm}|
    >{\columncolor{tablecolor2}}>{\centering\arraybackslash}p{2.1cm}|
    >{\columncolor{tablecolor3}}>{\centering\arraybackslash}p{2.1cm}|
    >{\columncolor{tablecolor4}}>{\centering\arraybackslash}p{2.1cm}|
    >{\columncolor{tablecolor5}}>{\centering\arraybackslash}p{2.5cm}|
    >{\columncolor{tablecolor9}}>{\centering\arraybackslash}p{2.2cm}|
}
\hline
\textbf{Method} & \textbf{Factor 1*} & \textbf{Factor 2*} & \textbf{Factor 3*} & \textbf{Total CEV} & \textbf{Factor Balance} \\
\hline

Optimal Affine Linear & 
\textbf{0.126} & 0.137 & 0.080 & \textbf{0.342} & 0.586 \\
\hline

Trained Affine Linear & 
$0.05 \pm 0.05$ & $0.05 \pm 0.04$ & $0.05 \pm 0.05$ & $0.15 \pm 0.05$ & \textbf{0.885} \\
\hline

Trained Nonlinear & 
$0.07 \pm 0.07$ & $0.07 \pm 0.06$ & $0.07 \pm 0.05$ & $0.18 \pm 0.07$ & 0.688 \\
\hline

PCA & 
0.109 & \textbf{0.144} & 0.080 & 0.333 & 0.555 \\
\hline

FF3 & 
\textbf{0.126} & 0.018 & \textbf{0.103} & 0.248 & 0.146 \\
\hline

\end{tabular}
\end{table}

The Trained Nonlinear autoencoder had the highest MSE of all the models, indicating that more complex transformations overfit the training data and provides suboptimal accuracy in reconstruction. The Trained Affine Linear autoencoder also had a similarly high MSE while PCA and the Optimal Affine Linear model yield much lower MSEs, with the Optimal Affine Linear model providing the lowest MSE. 

Both the Trained Nonlinear and Trained Affine Linear autoencoders exhibit a relatively high factor balances, low CEVs, and high uncertainty values, suggesting that they both learn a small amount of idiosyncratic noise instead of market patterns. The relatively high uncertainty of the variances show that both autoencoders' latent spaces are not only weak but unstable. Furthermore, the nonlinearity of the Trained Nonlinear autoencoder does not appear to offer any advantages in capturing meaningful underlying structure in the return data. 

The Optimal Affine Linear and PCA models both perform similarly across CEV and factor balance metrics. These models both reach higher levels of CEV than FF3 factors from the French Data Library. Furthermore, the Optimal Affine Linear model's latent space of true market data provides the highest CEV, indicating that the optimal transformation also provides a meaningful and economically relevant compression of our stock universe. More generally, this demonstrates that the Optimal Affine Linear Autoencoder mapping provides an accurate reconstruction of input data and an interpretable and meaningful latent space, which many nonlinear models fail to do.

\subsection{Nonlinear Experiment: Shallow Water Equations}\label{swe}
In this section, we present an application of our optimal linear inverse end-to-end map, and also present a comparison to a nonlinear learned inverse map. The forward problem we consider is the evolution of the shallow water equations (SWEs) in time, with the initial conditions being the input signal $\mathbf{x}$ and the system state at some later time $t_n$ being the noisy observations, $\mathbf{y}$. Building on works such as \textcite{jo2019deep}, which demonstrate the use of neural networks to approximate differential equations in both forward and inverse settings, we integrate our theoretical results on rank-constrained optimal networks into this framework for partial differential equation (PDE) approximation. We observe that the optimal linear inverse mapping performs significantly better than a learned nonlinear mapping, showing the efficacy of our linear result in a nonlinear context.

\paragraph{Overview and Data Generation}
The SWEs are a set of three nonlinear coupled PDEs that describe the dynamics of a thin layer of fluid of constant density in hydrostatic balance. \Cref{app:SWE} outlines the full equations, along with a detailed explanation of all the variables and the data collection process. To generate the training dataset, we simulate the SWEs 2,500 times for four different types of initial conditions and then extract the system states for the 0-th and the 1,500-th time-step, which corresponds to $t=0s$ and $t\approx(1,\!500\times 50)s$. A system state (each data point $\mathbf{x}_t$) consists of the values of three variables over a two-dimensional mesh (of size $64\times64)$: the $\chi_1$ and $\chi_2$ velocities, referred to as $u$ and $v$, and the deviation from mean depth (surface height), referred to as $\eta$. This generates a dataset of 10,000 instances, each with a signal datapoint $\mathbf{x}_0$ and an observation data point $\mathbf{x}_{1,500}$. \Cref{fig:initial_conditons_example} shows the different possible types of initial conditions in $u$, $v$, and $\eta$. Lastly, we add Gaussian white noise to the observation data $\mathbf{x}_{1,500}$ (in a similar fashion to the process outlined in \Cref{sec:medMNIST_dataset_overview}), setting up the inverse problem as 

\[ 
\mathbf{x}_{1,500}=\mathbf{F}(\mathbf{x}_{0})+\bm{\varepsilon},
\]
where we hope to find a mapping $\mathbf{A}$ that recovers the initial conditions $\mathbf{x}_{0}$ given the system state $\mathbf{x}_{1,500}$.

\begin{figure}
    \centering
    \includegraphics[width=1.0\linewidth]{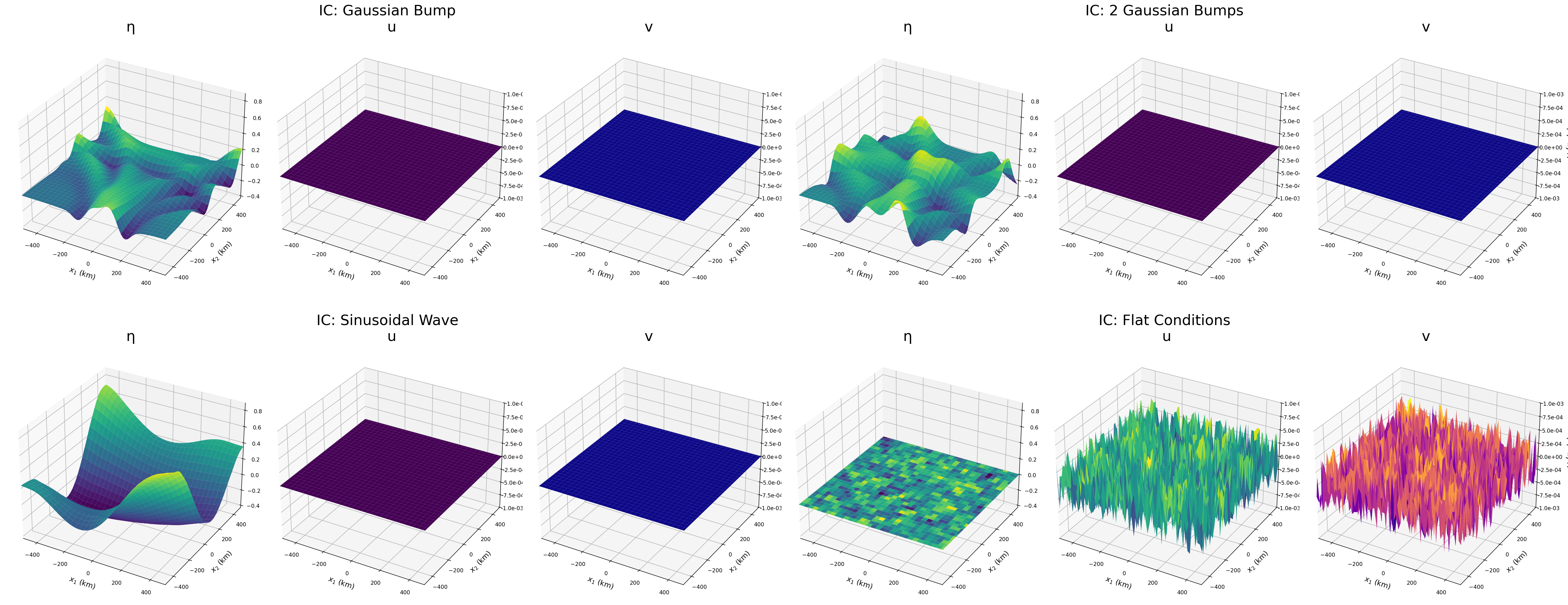}
    \caption{Visualization of the four different kinds of initial conditions used in the dataset. Note that the z axis denotes $\eta$ (m) for the first and fourth columns, $u$ (m/s) for the second and fifth columns and $v$ (m/s) for the others columns.}
    \label{fig:initial_conditons_example}
\end{figure}

\paragraph{Optimal Linear Approach.} The first straightforward approach to this problem is to assume that the forward mapping $\mathbf{F}$ is linear, and thus use the optimal inverse end-to-end mapping derived in \Cref{th:e2eFwd}. Although we know a priori that the SWEs are nonlinear, this serves as a solid test of the efficacy of our theoretical results on real-world, nonlinear data.

Recall from \Cref{th:e2eFwd} that the optimal linear inverse end-to-end mapping is given by $\mathbf{\widehat{A}} = (\mathbf{\Gamma}_X\mathbf{F}^\top \mathbf{L}_Y^{\dagger \, \top})_r \mathbf{L}_Y^{\dagger}$, where $r= \operatorname{rank}(\mathbf{\widehat{A}})$. To compute model predictions, we: 
\begin{itemize}
    \item Vectorize each data point from $\mathbf{x}_t\in \mathbb{R}^{3\times64\times64}$ into a column vector in $\mathbb{R}^{12,288}$, and concatenate data points into signal and observation matrices $\mathbf{X}, \mathbf{Y} \in \mathbb{R}^{12,288\times10,000}$. 
    
    \item Use empirical estimates for $\mathbf{\Gamma}_X \mathbf{F}^\top$ and $\mathbf{\Gamma}_Y$. For random variables $X$ and $Y$ defined as in \Cref{th:e2eInv}, we have the identity
    \[
        \mathbb{E}[XY^\top] = \mathbf{\Gamma}_X \mathbf{F}^\top.
    \]
    Given our data matrices $\mathbf{X}$, $\mathbf{Y}$, we approximate this expectation by the empirical cross-covariance
    \[
        \mathbf{\Gamma}_X \mathbf{F}^\top \approx \tfrac{1}{J} \mathbf{X} \mathbf{Y}^\top.
    \]
    Similarly, we approximate the second moment of $Y$ empirically with
    \[  \mathbf{\Gamma}_\mathbf{Y} = \tfrac{1}{J} \mathbf{Y} \mathbf{Y}^\top + 10^{-2} \mathbf{I}_{12,288}
        = \mathbf{L}_\mathbf{Y} \mathbf{L}_\mathbf{Y}^\top,
    \]
    where the change in subscript again denotes that we are working with empirical data. 
    We use the ridge term $10^{-2} \mathbf{I}_{12,288}$ to ensure that the empirical second moment is SPD, permitting the Cholesky decomposition.

    \item Compute the matrices $\mathbf{L}_\mathbf{Y}^{-1}$ and $\mathbf{L}_\mathbf{Y}^{-\top}$, and form the rank-$r$ optimal mapping 
    \[
        \widehat{\mathbf{A}}  
        = \left(\mathbf{\Gamma}_X\mathbf{F}^\top \mathbf{L}_Y^{-\top}\right)_r \mathbf{L}_Y^{-1}
        \approx \left(\mathbf{X} \mathbf{Y}^\top \, \mathbf{L}_\mathbf{Y}^{-\top} \right)_r \mathbf{L}_\mathbf{Y}^{-1}.
    \]

    \item Compute the approximation $\widetilde{\mathbf{X}} = \mathbf{\widehat{A}}\mathbf{Y}$. 

\end{itemize}

To compare results, we compute errors across each of the three variables ($u$, $v$, and $\eta$) separately (i.e., $\mathbf{X}_u$, $\mathbf{X}_v$, and $\mathbf{X}_\eta$ denote the corresponding rows of $\mathbf{X}$ associated with each variable). For $\eta$, we use a standard error metric of normalized root mean squared error (NRMSE), however for $u$ and $v$, we use mean absolute error (MAE). Using a different loss metric for these two variables is necessary as 75\% of the values of $u$ and $v$ in $\mathbf{X}$ are zero, resulting in error metrics like NRSME reflecting large and inaccurate errors. Additionally, we compute NRMSE across all three variables to have a single loss metric that is representative of the errors in all three variables. To verify that the trained models are genuinely learning the inverse mapping rather than simply memorizing the training data, we evaluate their performance on two distinct testing sets. The in-distribution set contains 2,000 samples generated with the same four types of initial conditions as the training data. In contrast, the out-of-distribution set consists of 1,000 samples derived from two previously unseen types of initial conditions (of ring wave and step wave).

Experiments with different values of $r$ demonstrated that the minimal RMSE across the entire dataset was obtained at $r\approx250$, thus we use $r=250$ as a benchmark rank for all our experiments. \Cref{tab:sw_error_comparison} shows error metrics for the linear model. We also present a visualization of these results in \Cref{fig:sw_reconstructions}.

\begin{table}
\caption{Definition of error metrics and error values for optimal linear and learned nonlinear models across two different testing sets, with $r=250$. The in-distribution testing set contains 2,000 data points with the same initial conditions as the training set, and the out-of-distribution testing set contains 1000 data points with the two unseen initial conditions.
Note: $\|\cdot\|_1$ is the element-wise $\ell_1$ (Manhattan) norm.}
\centering\small
\renewcommand{\arraystretch}{1.8}
\begin{tabular}{|
    >{\columncolor{tablecolor1}}>{\centering\arraybackslash}p{0.1\linewidth}|
    >{\columncolor{tablecolor2}}>{\centering\arraybackslash}p{0.15\linewidth}|
    >{\columncolor{tablecolor3}}c|
    >{\columncolor{tablecolor3}}c|
    >{\columncolor{tablecolor4}}c|
    >{\columncolor{tablecolor4}}c|
}
\hline
\textbf{Error} & \textbf{Definition} &
\multicolumn{2}{>{\columncolor{tablecolor3}}c|}{\textbf{In-Distribution Testing Set}} &
\multicolumn{2}{>{\columncolor{tablecolor4}}c|}{\textbf{Out-of-Distribution Testing Set}} \\
\cline{3-6}
\textbf{Metric} & &
\textbf{Optimal} & \textbf{Learned} &
\textbf{Optimal} & \textbf{Learned} \\
\hline
\rule{0pt}{4.7ex}Total NRMSE & 
 $\dfrac{\|\mathbf{X}-\widetilde{\mathbf{X}}\|_{\mathrm F}}
       {\|\mathbf{X}\|_{\mathrm F}}$ & 
 $0.18$ & $0.26$ & $0.12$ & $0.87$ \rule[-4ex]{0pt}{0pt} \\
\hline
\rule{0pt}{4.7ex}$\eta$ NRMSE & 
 $\dfrac{\|\mathbf{X}_\eta-\widetilde{\mathbf{X}}_\eta\|_{\mathrm F}}
       {\|\mathbf{X}_\eta\|_{\mathrm F}}$ & 
 $0.18$ & $0.26$ & $0.12$ & $0.87$ \rule[-4ex]{0pt}{0pt} \\
\hline
\rule{0pt}{4ex}$u$ MAE & 
$\displaystyle\tfrac1J\|\mathbf{X}_u-\widetilde{\mathbf{X}}_u\|_1$ & 
$6.3\times10^{-5}$ & $6.4\times10^{-5}$ & $1.8\times10^{-6}$ & $5.0\times10^{-4}$
\rule[-2.2ex]{0pt}{0pt} \\
\hline
\rule{0pt}{4ex}$v$ MAE & 
$\displaystyle \tfrac1J \|\mathbf{X}_v-\widetilde{\mathbf{X}}_v\|_1$ & 
$6.3\times10^{-5}$ & $6.4\times10^{-5}$ & $1.7\times10^{-6}$ & $5.0\times10^{-4}$
\rule[-2.2ex]{0pt}{0pt} \\
\hline
\end{tabular}
\label{tab:sw_error_comparison}
\end{table}

\paragraph{Learned Nonlinear Approach.} Although the linear model is successful at this inverse problem, we know a priori that the forward process here is nonlinear and thus we also try to learn the inverse mapping with a nonlinear end-to-end deep neural network. This provides a solid benchmark comparison to the performance of our optimal linear model.

We use a fully connected neural network with an encoder-decoder architecture and bottleneck rank of $r=250$. Both the encoder and decoder have three linear layers with nonlinear ReLU activations stacked in between. Additionally, we use the same normalized data, the \texttt{ADAM} optimizer, MSE loss, and a batch size of 80 to train the nonlinear network. With some experimentation we see that after 100 epochs of training, the MSE loss converged close to a minimum. We report the error values across testing sets in \Cref{tab:sw_error_comparison} and present visualizations of the results in \Cref{fig:sw_reconstructions}.

\begin{figure}
\centering
\begin{tabular}{c}
    \includegraphics[width=0.9\linewidth]{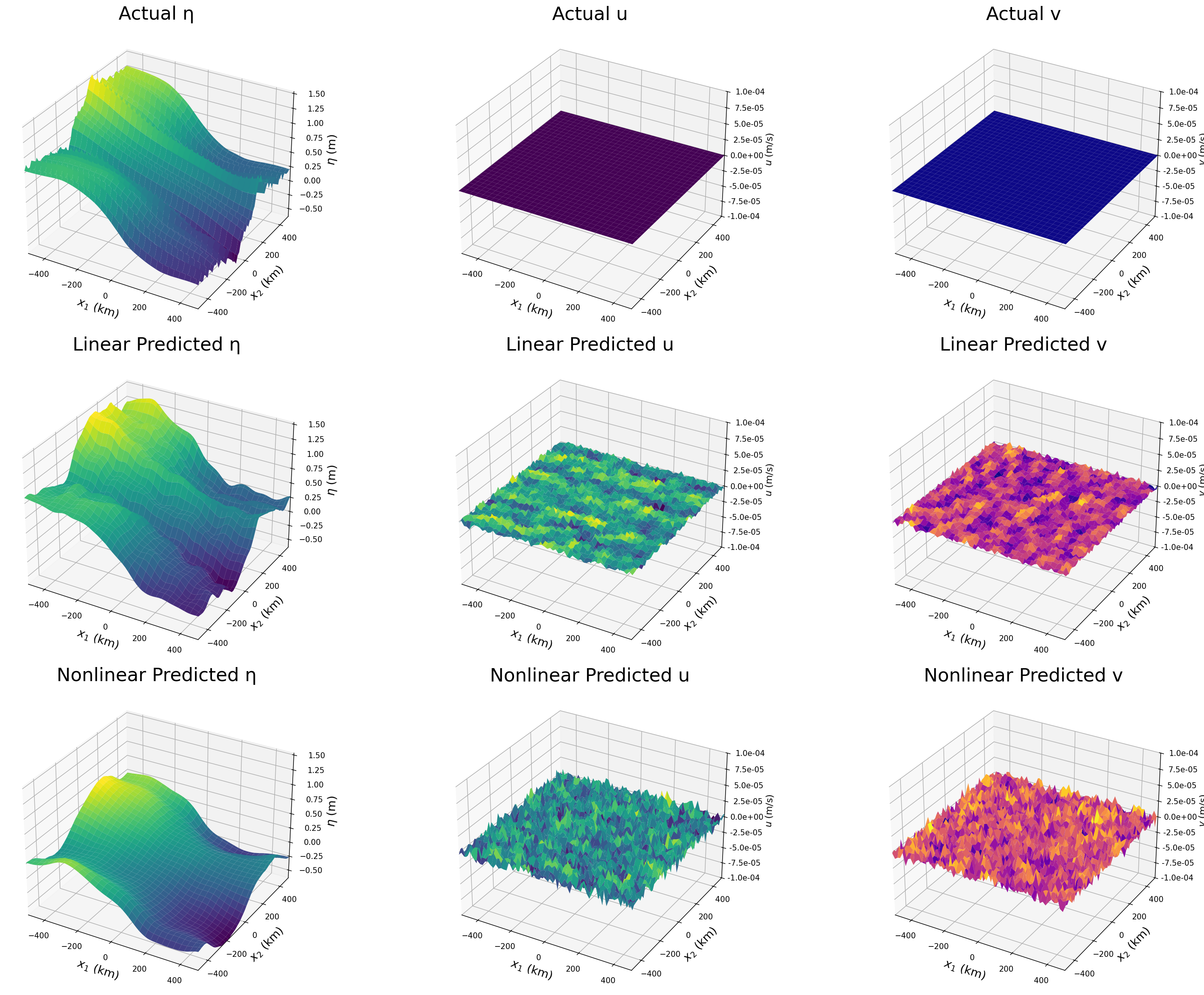} 
\end{tabular}
\caption{Example of actual initial conditions, linear model reconstructions, and nonlinear learned model reconstructions for all three variables, for an instance from the out-of-distribution testing set.}
\label{fig:sw_reconstructions}
\end{figure}

\paragraph{Results.} \Cref{tab:sw_error_comparison} evidently shows that the optimal linear inverse mapping performs significantly better at solving the inverse problem than a learned nonlinear mapping, despite the forward process being nonlinear. One possible explanation is that the simulated forward process retains some approximate linearity, as the solver assumes the momentum equations to be linear. However, we verify that the linear model does not just ``memorize'' the data by testing it on unseen data and initial conditions, and we see that it still performs significantly better. While nonlinearity alone does not guarantee better approximations for nonlinear processes, alternative nonlinear architectures may still outperform the optimal linear mapping in this setting. While the Universal Approximation Theorem \cite{hornik1989multilayer} guarantees that such a mapping exists, finding it is not trivial and in practice demands extensive architectural design, hyperparameter tuning, and careful selection of stochastic optimization procedures. This result emphasizes the need for using well-understood linear mappings as baselines for comparison in scientific ML tasks, as well as the efficacy of our optimal linear mappings for solving even nonlinear inverse problems with superb accuracy while remaining computationally efficient.

\section{Discussion and Conclusions}\label{conclusion}
In this work, we provide general closed-form solutions of optimal rank-constrained mappings for linear encoder-decoder architectures through the lens of Bayes risk minimization. Our main contribution is the derivation of closed-form, rank-constrained optimal estimators across a diverse set of scientific modeling scenarios, including forward modeling, inverse recovery, and special cases like autoencoding and denoising. We extend these formulations to accommodate rank-deficiencies in the parameter distributions, measurement processes, and forward operators, offering a generalized theory applicable to realistic, imperfect data settings.

To verify our results, we provide several numerical experiments in the form of toy medical imaging via the \texttt{MedMNIST} dataset, in which we highlight how our derived optimal mappings outperform those learned by their encoder-decoder network counterpart. We perform an in-depth comparison of our optimal mapping in a financial setting to the standard models for factor analysis. Our results demonstrate the power of the simple linear transformation to accurately identify the latent factors within heteroskedastic asset returns. Our transformation performs better than PCA and greatly outperforms the other learned linear and nonlinear autoencoder architectures. Finally, we test the limits of our optimal mappings' scope by applying them to a challenging nonlinear inverse problems via the SWEs, in which our optimal mapping even outperforms a deep nonlinear neural network, at predicting initial conditions given the system's future state.

Several extensions naturally emerge from our study. Our work can easily be extended to using the $R$-weighted two norm as the loss function, while still drawing on results from \textcite{friedland2007generalized}. While we focus on linear architectures, extending our Bayes risk minimization framework to nonlinear encoder-decoder networks could offer insight into more expressive models while retaining partial analytical tractability. Further, our framework is grounded in expected risk but does not quantify estimator uncertainty. Future work may explore closed-form expressions for the posterior variance of optimal estimators or derive confidence bounds under noise models. To improve the practical applicability of these results to real-world datasets, future work could explore randomized SVD methods for scenarios where the data is high-dimensional and computing the traditional SVD is infeasible. In dynamic settings where data streams evolve over time, it would be valuable to extend our results to adaptive estimators that update rank-constrained mappings in an online fashion.

For our numerical results, our ongoing work in financial modeling aims to further emphasize the optimal affine linear autoencoder's efficacy in factor analysis by using Shapley values as a machine learning metric \cite{shapley1953value} to better interpret the economic significance of the latent space in the optimal affine linear autoencoder. Preliminary results using Shapley values demonstrate a distinct but more economically complex latent space than expected. This does not reduce the importance of the optimal affine linear autoencoder's result, but rather hints that it is learning a different and more representative latent space. Continuing experiments aim to investigate and interpret this result further.

\section*{Availability of Data and Materials}
The data and code used in this work are available on \href{https://github.com/alexdelise/CMDS-REU-MADDI}{GitHub}.

\section*{Acknowledgment}
The authors acknowledge support from the National Science Foundation Division of Mathematical Sciences under grant No. DMS-2349534.

\printbibliography

\appendix

\section{Special Case \texttt{MedMNIST} Figures}\label{app:medmnistFigs}
In this appendix we provide some analysis on the special cases of the forward and inverse end-to-end problem, namely the noiseless autoencoding and data denoising problem. We omit the figures and analysis of all affine linear mappings as their results match their linear counterparts.

\paragraph{Noiseless Autoencoder.} 
In the noiseless autoencoder setting, the objective is to compress and reconstruct the original signal $\mathbf{x}$ using a rank-constrained linear autoencoder that approximates the identity map.

\Cref{fig:noiselessFullComp} shows the average per-sample $\ell_2$ reconstruction error across rank-$r$ bottlenecks for classic noiseless autoencoder on the four selected \texttt{MedMNIST} datasets. Here, the optimal mappings improve steadily and converge toward optimal signal reconstruction as $r$ increases. In contrast, the learned mappings remain largely flat and fail to exploit the added model capacity, resulting in a persistent and often widening gap from optimal performance. Unlike in the forward and inverse end-to-end scenarios, the reconstruction error for the optimal mappings does not plateau as the rank increases. This is because in the absence of noise, higher-rank mappings may continue to capture increasingly fine-grained structure in the data distribution, and the optimal solution converges to the identity map as $r \to n$.

\Cref{fig:noiselessOptVSLearned} visualizes the gap between the optimal and learned mappings for a representative sample from \texttt{ChestMNIST} at rank $r = 25$. The optimal mapping for the linear autoencoder yields a visually faithful reconstruction and low error, while the learned mapping introduces notable grainy artifacts and is unable to learn the finer anatomical structure of the original image. The associated error map confirms these structured discrepancies in the learned output that are not present in the optimal case.

\begin{figure}
    \centering
    \includegraphics[width=0.67\linewidth]{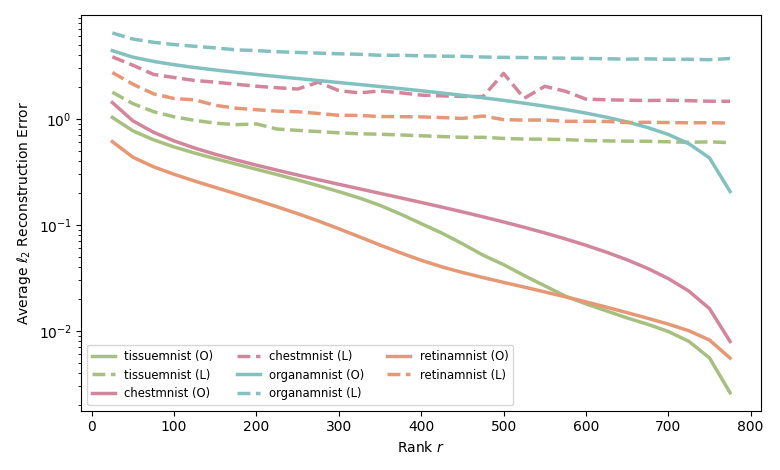}
    \caption{Average per-sample $\ell_2$ reconstruction error versus bottleneck rank $r$ across \texttt{MedMNIST} datasets for the classic noiseless autoencoder problem, where the goal is to learn a rank constrained identity mapping from $\mathbf{x}$ to $\mathbf{x}$. The optimal and learned mappings minimize the empirical losses $\tfrac{1}{J} \sum_{j=1}^{J} \| \widehat{\mathbf{A}}_r \mathbf{x}_j - \mathbf{x}_j \|_2^2$ and $\tfrac{1}{J} \sum_{j=1}^{J} \| \mathbf{A}_r \mathbf{x}_j - \mathbf{x}_j \|_2^2$, respectively. Solid lines denote optimal mappings $\widehat{\mathbf{A}}_r$ (O), while dashed lines denote learned encoder-decoder mappings $\mathbf{A}_r$ (L).}
    \label{fig:noiselessFullComp}
\end{figure}

\begin{figure}
    \centering
    \includegraphics[width=0.75\linewidth]{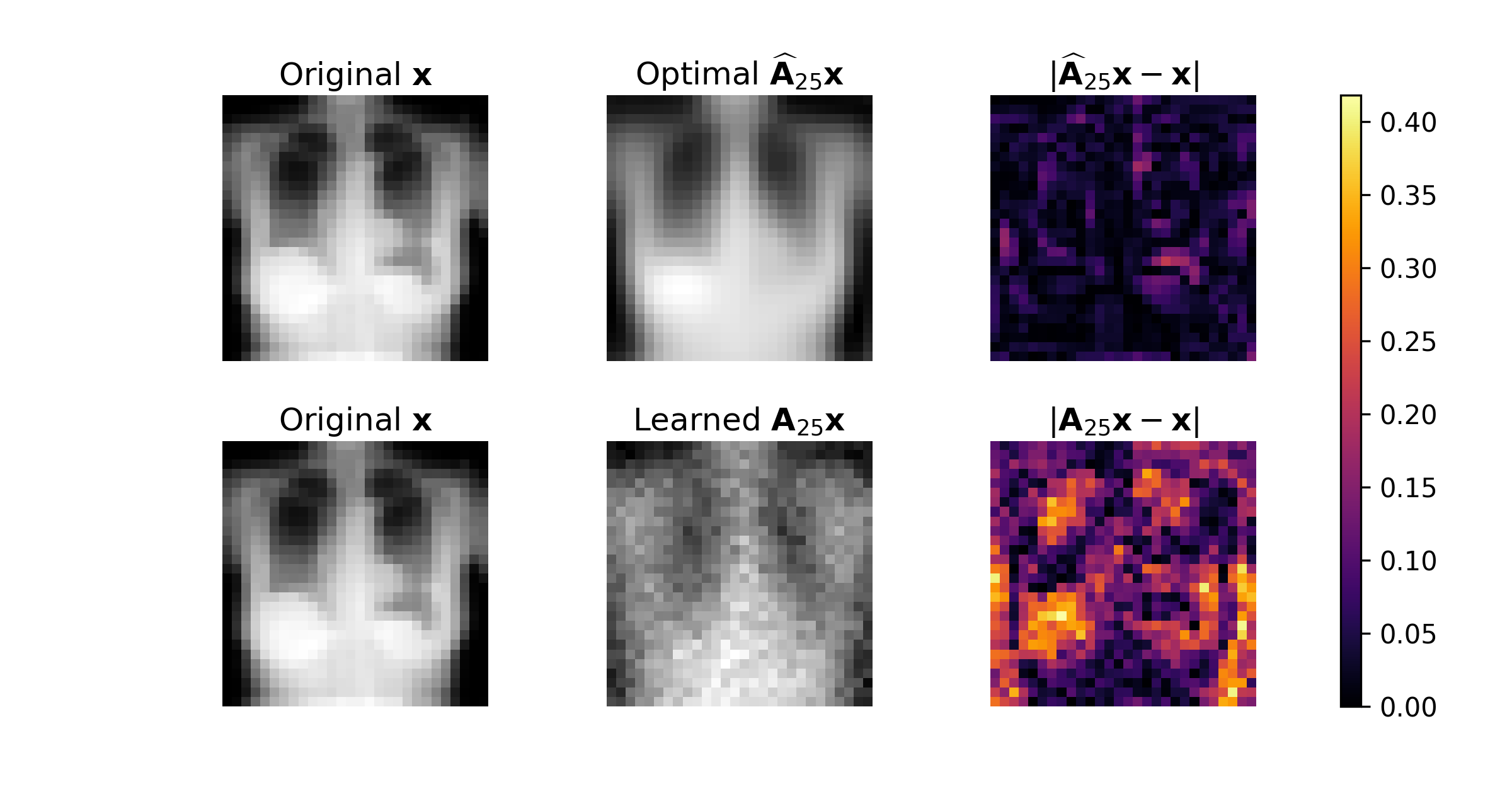}
    \caption{Example reconstruction results from \texttt{ChestMNIST} with bottleneck rank $r = 25$ for the classic linear noiseless autoencoder problem. \textbf{Top row:} Reconstructions using optimal mappings. Left to right: original signal $\mathbf{x}$; optimal reconstruction $\widehat{\mathbf{A}}_{25} \mathbf{x}$ and corresponding error map $| \widehat{\mathbf{A}}_{25} \mathbf{x} - \mathbf{x} |$. \textbf{Bottom row:} Learned autoencoder reconstructions. Left to right: original image $\mathbf{x}$; learned reconstruction $\mathbf{A}_{25} \mathbf{x}$ and corresponding error map $| \mathbf{A}_{25}\mathbf{x} - \mathbf{x} |$.}
    \label{fig:noiselessOptVSLearned}
\end{figure}

\paragraph{Denoising Autoencoder.} 
In the presence of noise, the objective shifts from exact reconstruction to denoising. Here, the encoder-decoder aims to learn a low-rank operator that suppresses noise while recovering the clean signal $\mathbf{x}$.

\Cref{fig:NoisyFullComp} summarizes the average $\ell_2$ reconstruction error across varying ranks $r$ and datasets. Like in the previous scenarios, the learned models exhibit little to no improvement in reconstruction with increasing rank and only somewhat track the trend of the optimal performance. This gap between the optimal and learned reconstructions highlights a failure to effectively learn from the data. In particular, the encoder-decoder pair struggles to denoise and reconstruct the input accurately. This issue is especially pronounced at lower ranks and in datasets with limited sample sizes, such as \texttt{RetinaMNIST}, where the learned mapping deviates significantly from the optimal solution.

\Cref{fig:noisyOptVSLearned} shows a representative sample from \texttt{ChestMNIST} at rank $r = 25$. The optimal mapping effectively restores the major anatomical features while suppressing noise, with little structural error. In contrast, the learned reconstruction retains many noise artifacts and exhibits even more structured error regions, particularly around high-contrast boundaries. This qualitative gap between the error of the optimal and learned mappings illustrates the continued advantage of using the optimal mapping to capture fine structural details in image reconstruction.

\begin{figure}
    \centering
    \includegraphics[width=0.67\linewidth]{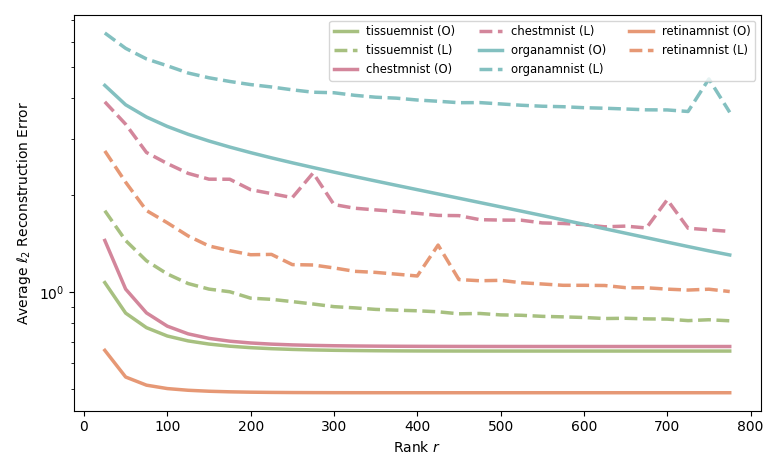}
    \caption{Average per-sample $\ell_2$ reconstruction error versus bottleneck rank $r$ across \texttt{MedMNIST} datasets in the line data denoising setting. The expressions $\tfrac{1}{J} \sum_{j=1}^{J} \| \widehat{\mathbf{A}}_r (\mathbf{x}_j + \bm{\varepsilon}_j) - \mathbf{x}_j \|_2^2$ and $\tfrac{1}{J} \sum_{j=1}^{J} \| \mathbf{A}_r (\mathbf{x}_j + \bm{\varepsilon}_j) - \mathbf{x}_j \|_2^2$ correspond to the optimal and learned reconstruction errors, respectively. Solid lines denote optimal mappings $\widehat{\mathbf{A}}_r$ (O), while dashed lines denote learned encoder-decoder mappings $\mathbf{A}_r$ (L).}
    \label{fig:NoisyFullComp}
\end{figure}

\begin{figure}
    \centering
    \includegraphics[width=0.75\linewidth]{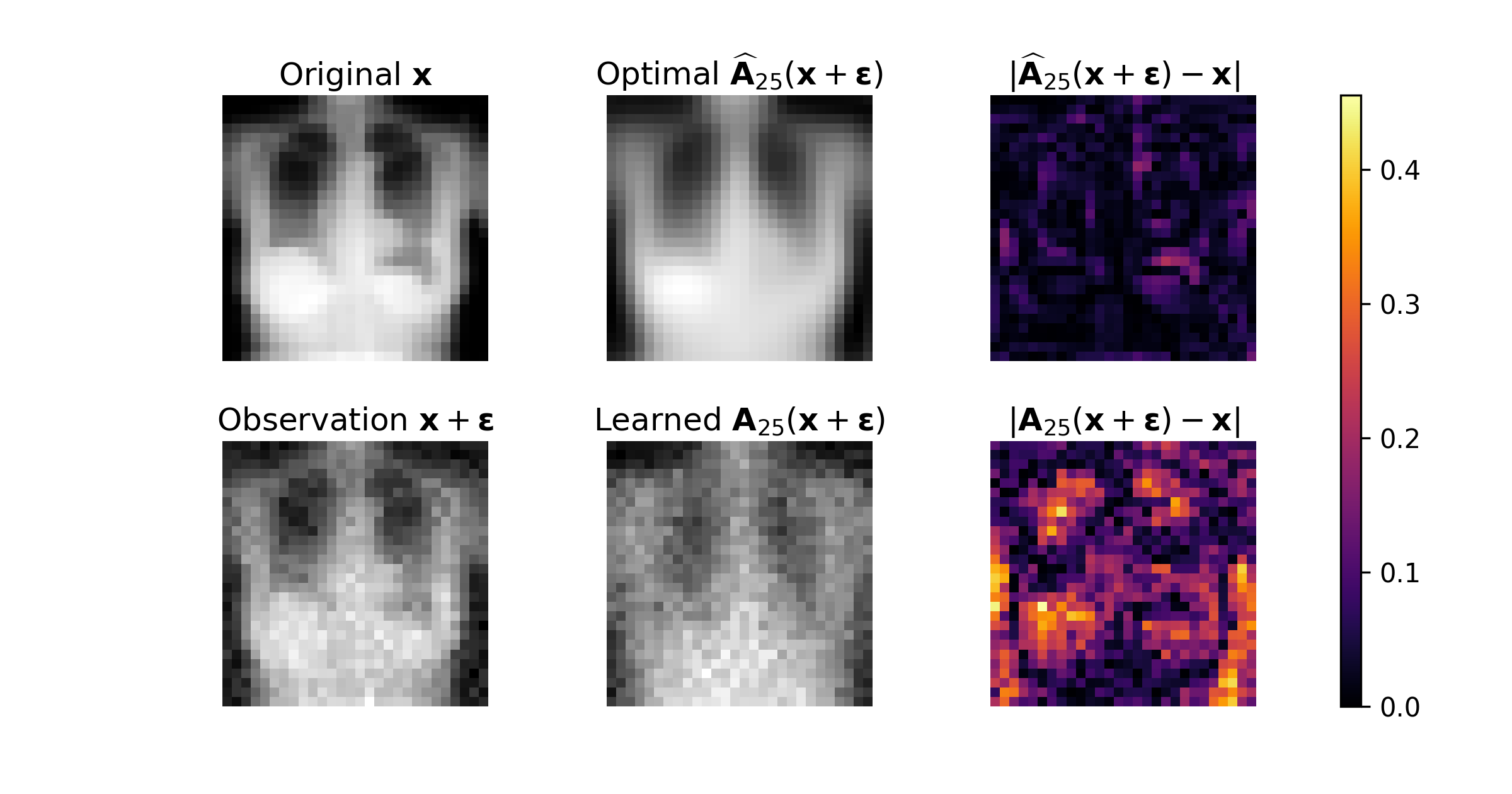}
    \caption{Example reconstruction results from \texttt{ChestMNIST} with bottleneck rank $r = 25$ in the linear data denoising setting. \textbf{Top row:} original ground-truth image $\mathbf{x}$; optimal denoised reconstruction $\widehat{\mathbf{A}}_{25}(\mathbf{x} + \bm{\varepsilon})$; and reconstruction error map $| \widehat{\mathbf{A}}_{25}(\mathbf{x} + \bm{\varepsilon}) - \mathbf{x} |$. \textbf{Bottom row:} noisy observation $\mathbf{x} + \bm{\varepsilon}$; learned encoder-decoder denoised reconstruction $\mathbf{A}_{25}(\mathbf{x} + \bm{\varepsilon})$; and corresponding reconstruction error map $| \mathbf{A}_{25}(\mathbf{x} + \bm{\varepsilon}) - \mathbf{x} |$.}
    \label{fig:noisyOptVSLearned}
\end{figure}

\section{Financial Information}\label{app:financialdefs}

In this section, we note particularly important definitions of financial terms and our model architectures referenced in \Cref{finance}. We also provide results for a numerical experiment that uses synthetic market data for factor analysis.

\subsection{Model Architectures}\label{fin-arc}
Given a financial dataset $\mathbf{X} \in \mathbb{R}^{T \times A}$, where $T$ denotes the number of time steps (e.g., trading days) and $A$ denotes the number of assets, we seek to learn a lower-dimensional representation $\mathbf{Z} \in \mathbb{R}^{T \times r}$, where $r \leq A$ is the number of latent factors. The latent representation is then used to reconstruct the original dataset, producing $\widetilde{\mathbf{X}} \in \mathbb{R}^{T \times A}$. Our trained model architectures are as follows:
\begin{itemize}
    \item \textbf{Trained Affine Linear Autoencoder:}
        \[
            \mathbf{X} \in \mathbb{R}^{T \times A} 
            \xrightarrow{\text{Encoder}} \mathbf{Z} \in \mathbb{R}^{T \times r} 
            \xrightarrow{\text{Decoder}} \widetilde{\mathbf{X}} \in \mathbb{R}^{T \times A},
            \]
            where $r = 3$ in our experiments. We randomly initialized the Trained Affine Linear autoencoder and trained it for 150 epochs. 
    
        \item \textbf{Trained Nonlinear Autoencoder:}
            \[
                \begin{aligned}
                \mathbf{X} \in \mathbb{R}^{T \times A} 
                &\xrightarrow{\text{Affine}} \mathbb{R}^{T \times H_1} 
                \xrightarrow{\text{Leaky ReLU}} \mathbb{R}^{T \times H_2} 
                \xrightarrow{\text{Affine}} \cdots 
                \xrightarrow{\text{Affine}} \mathbb{R}^{T \times r} \\
                &\xrightarrow{\text{Affine}} \cdots 
                \xrightarrow{\text{Leaky ReLU}} \mathbb{R}^{T \times H_2} 
                \xrightarrow{\text{Affine}} \mathbb{R}^{T \times H_1} 
                \xrightarrow{\text{Affine}} \widetilde{\mathbf{X}} \in \mathbb{R}^{T \times A},
                \end{aligned}
            \]
        where $H_1, H_2, \ldots$ denote intermediate hidden dimensions, and $r = 3$ is the latent dimension. All affine mappings are followed by Leaky ReLU activations with negative slope $a = -0.01$. The bottleneck layers of each autoencoder were three neurons wide, corresponding to the factors of the FF3 model. We randomly initialize the trained nonlinear autoencoder and train it for 100 epochs.
\end{itemize}
Both networks were trained with a batch size of 64 and with an 80/20 chronological train/test split on the data. Using an MSE loss, we train our model using the \texttt{ADAM} \cite{kingma2014adam} optimizer with a learning rate of $10^{-3}$.  As a final baseline for comparison, we also include PCA in our results for both synthetic and market data. 

\subsection{Fama-French Factors}\label{ff3-sect}
The Fama-French Three Factor model expresses the expected return of a single asset over a time period $t$ as 
$$
\mathcal{R} = \mathcal{R}_f + \gamma_{1}(\text{Market}) + \gamma_{2}( \text{SMB}) + \gamma_{3} (\text{HML}) + \alpha,
$$ describing a linear relation between an asset's expected return rate and its three FF3 factors: Market, SMB, and HML. Its factors and related terms, as noted in \textcite{fama1993common}, are defined as the following:
\begin{enumerate}
    \item $\mathcal{R}$ denotes the expected rate of return of an asset. 
    \item $\mathcal{R}_f$ denotes the risk-free return rate or the theoretical return of a zero-risk investment, constant across all assets.
    \item The Market Excess factor(Market) quantifies the extent at which a market portfolio's return exceeds $\mathcal{R}_f$.
    \item The Size factor $(\text{SMB})$ measures the extent at which small cap stocks outperform large cap stocks.
    \item The Value factor $(\text{HML})$ measures the extent at which value stocks outperform growth stocks.
    \item The $\alpha$ parameter denotes the mean excess return of not explained by the three Fama-French factors, which varies across assets. 
    \item $\gamma_{1}, \gamma_{2}, \gamma_{3}$ are model parameters that vary by asset. These parameters are usually used for calculating $\mathcal{R}$, and therefore are not a part of this work.
\end{enumerate} 

\subsection{Notable Processes}\label{processes}

\begin{enumerate}
    \item AR($1$), also known as an Autoregressive Process of Order 1, assumes that each term in a time series depends linearly on the preceding value and some error term. The time series can be described by the following:
    \[
        X_t = \beta X_{t-1} + \epsilon_t^2 
    \]
    where $X_t$ is a stationary series, $\beta$ is some correlation factor, and $\epsilon_t^2 \sim \mathcal{N}(0,\upsilon^2)$ is white noise \cite{brockwell2016introduction}. 
    
    \item GARCH($1,1$) generates a time series $X_t = \upsilon_t^2$ for the volatility squared (variance) of a time series, which can be described by the following:
    $$\upsilon^2_t = \omega + \alpha \epsilon_{t-1}^2 + \beta \upsilon_{t-1}^2$$ where $\upsilon_t^2$ represents the variance, $\epsilon_t$ represents residual error from the most recent estimate, $\omega$ represents the base level volatility, $\alpha$ represents the sensitivity to recent shocks, and $\beta$ represents the extent at which past volatility affects the present volatility. Additional parameters were initialized as $\omega = 0.01 $, $\alpha = 0.1$, and $\beta = 0.85$ as these values resemble those of the U.S.~stock market, measured empirically in works such as Bollerslev \cite{bollerslev1986generalized}. 
\end{enumerate}

\subsection{Synthetic Data Experiment} \label{synth-data-exp}
We conduct an additional experiment on a synthetic dataset to empirically validate our claims in \Cref{finance}. We assess reconstruction performance using the MSE metric and use correlation values to interpret the latent space. Across both studies, our results indicate that the optimal affine linear model provides a robust, interpretable, and competitive alternative for reconstruction and factor analysis.

\subsubsection{Synthetic Data Creation}
We generated a synthetic dataset over a time horizon of $T=2,\!000$ days across $A = 10$ different assets and $r = 3$ latent factors. Using GARCH(1,1) and Fama-French three-factor (FF3) model assumptions to assign realistic values to market, size, and value risks, we generated a latent factor matrix $\mathbf{C} \in \mathbb{R}^{T \times r}$ and a factor loading matrix $\mathbf{B} \in \mathbb{R}^{A \times r}$. Then, volatility levels, also generated by GARCH($1,1)$, were collected in the matrix
\[
\mathbf{P} = \begin{bmatrix}
\upsilon^2_{1,1}  & \ldots  & \upsilon^2_{1,A} \\
\vdots & \ddots & \vdots \\
\upsilon^2_{T,1} & \ldots & \upsilon^2_{T,A}
\end{bmatrix},
\]
where $\upsilon_{t,a}^2$ denotes the variance (volatility) of the $a$-th asset return at time $t$.

To simulate observation noise, we define a Gaussian matrix $\mathbf{Q} \in \mathbb{R}^{T \times A}$ whose entries are sampled independently as
\[
    \mathbf{Q}_{ij} \sim \mathcal{N}(0,1), \quad \text{independently for} \quad 
    1 \leq i \leq T, \; 1 \leq j \leq A.
\]
We then scale this noise by the asset-specific volatility levels from $\mathbf{P}$ to obtain the final noise matrix $\bm{\Delta} \in \mathbb{R}^{T \times A}$ using the Hadamard (element-wise) product:
\[
    \bm{\Delta} = \mathbf{P} \odot \mathbf{Q}.
\]
The complete synthetic dataset of asset returns is thus given by
\[
    \mathbf{X} = \mathbf{C}\, \mathbf{B}^\top + \bm{\Delta},
\]
where $\mathbf{C}$ encodes the latent factor time series, $\mathbf{B}$ contains the factor loadings, and $\bm{\Delta}$ introduces heteroskedastic noise consistent with market volatility.

\subsubsection{Synthetic Data Results} 
This section compares the MSE and aligned correlations of our Optimal Affine Linear autoencoder across the three different models outlined previously: a Trained Affine Linear autoencoder with random initialization, a Trained Nonlinear autoencoder with random initialization, and PCA. We note that we do not compare these results with the FF3 model here as synthetically generated data does not correspond to asset return data listed in the Kenneth French library.

To enable meaningful comparison, we account for the fact that the latent spaces in these autoencoders may not align with the true orientation of the latent factors. This necessitates an orthogonal rotation to ensure that each generated latent space is oriented in the direction that best matches the generated latent factors. To achieve this, we used a linear algebra approach called Orthogonal Procrustes Alignment (OPA), which aims to find an orthogonal rotation matrix that transforms one matrix into the best possible approximation of another matrix \cite{grave2018unsupervised, jain2021aligned}. While literature is scarce surrounding OPA in financial time series datasets, there is extensive documentation of its uses and benefits in other domains, such as neuroscience and psychology \cite{haxby2011common, haxby2020hyperalignment}. We use OPA to orient the latent space $\mathbf{Z} \in \mathbb{R}^{ J \times r}$ of each model in the direction that best matches the true factor loading matrix $\mathbf{B}$, while preserving orthogonality. This allows for an accurate and more interpretable comparison between $\mathbf{Z}$ and $\mathbf{B}$. OPA is performed after training and validation, but prior to factor analysis. Since the Trained Affine Linear and Trained Nonlinear models are much more sensitive to initialization and can get stuck in local minima, these models were trained and tested in a loop 100 times, as in the market example.

\begin{table}
\renewcommand{\arraystretch}{1.5}
\caption{Reconstruction performance comparison: MSE results for synthetic data.}
\label{tab:synthetic_mse}
\centering

\begin{tabular}{|
    >{\columncolor{tablecolor1}}>{\centering\arraybackslash}p{3.5cm}|
    >{\columncolor{tablecolor2}}>{\centering\arraybackslash}p{5.2cm}|
}
\hline
\textbf{Method} & \textbf{Mean Squared Error} \\
\hline

Optimal Affine Linear & 
$1.5\times 10^{-5}$
\\
\hline

Trained Affine Linear & 
$1.9\pm 0.1\times 10^{-5}$
\\
\hline

Trained Nonlinear & 
$2.6\pm 0.1 \times 10^{-5}$
\\
\hline

PCA & 
$1.6\times 10^{-5}$
\\
\hline

\end{tabular}
\end{table}

\begin{table}
\renewcommand{\arraystretch}{1.5}
\caption{Factor Recovery Performance: Post OPA Correlations for synthetic data. Items in bold have the highest correlation in their respective column. Factor names are labeled to match the corresponding synthetic factor. Factor 1 is approximate to the Market Factor, Factor 2 to SMB, and Factor 3 to HML. The true parameters for each asset/factor pair are found in factor loading matrix $\mathbf{B}$.}
\label{tab:synthetic_procrustes}
\centering

\begin{tabular}{|
    >{\columncolor{tablecolor1}}>{\centering\arraybackslash}p{3.5cm}|
    >{\columncolor{tablecolor2}}>{\centering\arraybackslash}p{2.5cm}|
    >{\columncolor{tablecolor3}}>{\centering\arraybackslash}p{2.5cm}|
    >{\columncolor{tablecolor4}}>{\centering\arraybackslash}p{2.5cm}|
}
\hline
\textbf{Method} & \textbf{Factor 1} & \textbf{Factor 2} & \textbf{Factor 3} \\
\hline

Optimal Affine Linear & 
\textbf{0.81} & \textbf{0.16} & 0.10 \\
\hline

Trained Affine Linear & 
$0.4 \pm 0.2$ & $0.12 \pm 0.06$ & $0.09 \pm 0.06$ \\
\hline

Trained Nonlinear & 
$0.4 \pm 0.3$ & $0.11 \pm 0.06$ & $\mathbf{0.12} \pm \mathbf{0.06}$ \\
\hline

PCA & 
\textbf{0.81} & 0.089 & 0.061 \\
\hline

\end{tabular}
\end{table}

\Cref{tab:synthetic_mse} shows the MSE values for the tested models, with standard deviations if applicable. \Cref{tab:synthetic_procrustes} displays the post-OPA correlation values by factor, with standard deviations if applicable.
The Trained Nonlinear autoencoder attains the highest reconstruction MSE of all models on validation data, indicating that its additional complexity does not translate into improved performance in this setting. The Trained Affine Linear autoencoder attains a high MSE as well, indicating that its flexibility is not enough to provide a high-accuracy reconstruction. PCA and the Optimal Affine Linear mapping both achieve comparably low MSEs, with the latter being the lowest. Overall, the Optimal Affine Linear solution offers reconstruction comparable to PCA while providing theoretical guarantees and interpretability that nonlinear models lack. These results confirm its viability and robustness in the presence of heteroskedastic noise.

Among the four models evaluated, the Trained Affine Linear autoencoder displays very weak correlations across all three latent factors, indicating that it struggles to capture the latent structure of the synthetic dataset. The Trained Nonlinear autoencoder also displays weak correlations across all factors, performing slightly better than the Trained Affine Linear model on Factors 2 and 3. 

In contrast, the Optimal Affine Linear mapping and PCA both achieve similar reconstruction MSEs. Factor 1, designed to mimic the Fama-French Market Excess Returns factor, the most prominent driver of asset returns \cite{fama1993common}, is reliably recovered across both models, with consistently high aligned correlations. Greater variation arises in the recovery of Factors 2 and 3, which were designed to represent the more subtle HML and SMB factors, respectively.

For these latter two factors, the Optimal and Trained Affine Linear models outperform PCA, with the Optimal Affine Linear model achieving the highest aligned correlations on both. The Trained Nonlinear model performs comparably on Factor 2 and surpasses all models on Factor 3, suggesting that its nonlinearity enables it to extract weaker economic signals that linear models may miss. However, this comes at the cost of reduced reconstruction accuracy and weak correlation on Factor 1.  While PCA remains effective at identifying dominant factors, the Optimal Affine Linear and Nonlinear models exhibit advantages in recovering subtler latent structure. Notably, the Optimal Affine Linear model achieves these gains without the added complexity of nonlinear training, demonstrating its ability to recover economically meaningful signals in a more interpretable and efficient manner.


\section{SWEs: Background and Data Generation}\label{app:SWE}
Here we provide supplementary information for numerical experients conducted with the shallow water equations in \Cref{swe}.

The full SWEs model, as referenced from \textcite{vallis2017aofd}, is given by
\begin{align}
\frac{\partial \eta}{\partial t} 
+ \frac{\partial}{\partial \chi_1}\left[(\eta + H)u\right] 
+ \frac{\partial}{\partial \chi_2}\left[(\eta + H)v\right] 
&= \sigma - w \label{eq:continuity}, \\[0.2cm]
\frac{\partial u}{\partial t} 
+ u \frac{\partial u}{\partial \chi_1} 
+ v \frac{\partial u}{\partial \chi_2} 
- f v 
+ g \frac{\partial \eta}{\partial \chi_1 }  
+ \kappa u 
&= \frac{\tau_{\chi_1}}{p_0 h}, \qquad\text{and}\label{eq:u_momentum} \\[0.2cm]
\frac{\partial v}{\partial t} 
+ u \frac{\partial v}{\partial \chi_1} 
+ v \frac{\partial v}{\partial \chi_2} 
+ f u 
+ g \frac{\partial \eta}{\partial \chi_2}  
+ \kappa v 
&= \frac{\tau_{\chi_2}}{p_0 h}, \label{eq:v_momentum}
\end{align}
where \Cref{eq:continuity} is known as the continuity equation, and \Cref{eq:u_momentum} and \Cref{eq:v_momentum} are known as the momentum equations. Here, $\chi_1$ and $\chi_2$ refer to the two spatial directions. \Cref{table:swes} provides a description of the relevant model variables.

\begin{algorithm}
\caption{Finite Differences Method}
\label{alg:SWE_solving}
\begin{algorithmic}[1]
\FOR{each time step $t$}
    \STATE Compute preliminary approximation of $u^{t+1}$ and $v^{t+1}$, using forward-in-time and forward-in-space schemes and
    neglecting the Coriolis term ($f$).
    \STATE Update $u^{t+1}$ and $v^{t+1}$ to include the effect of the Coriolis term, using the preliminary estimates.
    \STATE Use final $u^{t+1}$ and $v^{t+1}$ to determine upwind directions across the mesh
    \STATE Compute corresponding spatial derivatives based on upwind scheme.
    \STATE Compute $\eta^{t+1}$.
\ENDFOR
\end{algorithmic}
\label{alg:finiteDiff}
\end{algorithm}

\begin{table}[H]
\renewcommand{\arraystretch}{1.5}
\caption{\label{table:swes} Description of variables used in the SWEs \Cref{eq:v_momentum,eq:u_momentum,eq:continuity}}
\label{tab:SWE_variables}
\centering

\begin{tabular}{|
    >{\columncolor{tablecolor1}}>{\centering\arraybackslash}p{1.5cm}|
    >{\columncolor{tablecolor2}}>
    {\centering\arraybackslash}p{8.5cm}|
    >{\columncolor{tablecolor3}}>{\centering\arraybackslash}p{2.5cm}|
}
\hline
\textbf{Symbol} & \textbf{Description} & \textbf{Units} \\
\hline

$u$, $v$ & Velocity components in $\chi_1$ and $\chi_2$ directions & m/s \\
\hline

$\eta$ & Free surface elevation (deviation from mean depth) & m \\
\hline

$H$ & Total fluid depth & m \\
\hline

$f$ & Time-varying Coriolis parameter ($f = f_0 + \beta y$) & 1/s \\
\hline

$g$ & Gravitational acceleration & m/s\textsuperscript{2} \\
\hline

$\tau_{\chi_1}$, $\tau_{\chi_2}$ & Wind stress components & N/m\textsuperscript{2} \\
\hline

$\rho_0$ & Fluid density & kg/m\textsuperscript{3} \\
\hline

$\kappa$ & Bottom friction coefficient & 1/s \\
\hline

$\sigma$ & Surface mass flux (precipitation $-$ evaporation) & m/s \\
\hline

$w$ & Vertical velocity at the bottom & m/s \\
\hline

\end{tabular}
\end{table}

To generate data for our model, we simulate the forward time-evolution of a simplified version of the full SWEs. We make these assumptions to simplify the computation of the forward process while still preserving nonlinearity through the continuity equation. Specifically, we neglect friction, wind stress, and external sources or sinks by setting $\tau_x = 0$, $\tau_y = 0$, $\kappa = 0$, and $\sigma = 0$.
We further simplify the momentum equations by dropping the nonlinear advection term $(\mathbf{v} \cdot \nabla)\mathbf{v}$, retaining only the linear term $\partial \mathbf{v} / \partial t$. Here, $\mathbf{v} = (u, v)$ is the velocity field and $\nabla = \left( \frac{\partial}{\partial \chi_1}, \frac{\partial}{\partial \chi_2} \right)$ denotes the spatial gradient. This removes self-advection effects, allowing us to isolate the influence of external forces such as gradients in fluid height.
The simplified equations we thus obtain are
\begin{align}
\frac{\partial \eta}{\partial t} + \frac{\partial}{\partial \chi_1}\left[(\eta + H)u\right] + \frac{\partial}{\partial \chi_2}\left[(\eta + H)v\right] &= w, \\
\frac{\partial u}{\partial t} - f v + g \frac{\partial \eta}{\partial \chi_1}  &= 0, \qquad \text{and}  \\
\frac{\partial v}{\partial t} + f u + g \frac{\partial \eta}{\partial \chi_2}  &= 0.
\end{align}

Other important considerations in solving the equations include that the simulated forward model is stable only when the space and time discretizations obey the Courant-Friedrichs-Lewy (CFL) condition \cite{courant1928cfl}, which ensures that the time discretization is small enough to restrict the wave movement to only one spatial cell per time step. For the SWEs with wave speed $\sqrt{gH}$, the CFL condition becomes
\[
    \Delta t \leq \dfrac{\min\{\Delta\chi_1, \Delta\chi_2\}}{\sqrt{gH}}
    \quad \text{and} \quad
    \alpha \ll 1,
\]
where $\alpha = \Delta t \cdot f$.
Here, $\Delta t$ refers to the size of the time discretizations, and $\Delta\chi_1$ and $\Delta\chi_2$ refer to the sizes of the space discretizations in $\chi_1$ and $\chi_2$, respectively.

For our simulation, we use a $64 \times 64$ mesh over a $10^6 \times 10^6 \, \text{m}^2$ domain, with a time step given by $\Delta t = 0.1\, \min\{\Delta\chi_1, \Delta\chi_2\} / \sqrt{gH}$ to obey the CFL condition (which gives $\Delta t \approx 50\,\mathrm{s}$). Values of the other parameters used in our simulations are listed in \Cref{tab:variable_values}. To solve these equations using finite differences, we use the process summarized in \Cref{alg:finiteDiff}.

\begin{table}[H]
\renewcommand{\arraystretch}{1.5}
\caption{Physical and numerical parameter values used in the SWE simulations}. 
\label{tab:variable_values}
\centering

\begin{tabular}{|
    >{\columncolor{tablecolor1}}>{\centering\arraybackslash}p{2.2cm}|
    >{\columncolor{tablecolor2}}>
    {\centering\arraybackslash}p{6.5cm}|
    >{\columncolor{tablecolor3}}>{\centering\arraybackslash}p{3.0cm}|
}
\hline
\textbf{Variable} & \textbf{Description} & \textbf{Value} \\
\hline

$g$ & Acceleration due to gravity & $9.8$ m/s\textsuperscript{2} \\
\hline

$H$ & Depth of fluid & $100$ m \\
\hline

$f_0$ & Fixed part of Coriolis parameter & $10^{-4}$ 1/s \\
\hline

$\beta$ & Time-varying part of Coriolis parameter & $2 \times 10^{-11}$ 1/s \\
\hline

$\Delta\chi_1, \ \Delta\chi_2$ & Spatial Discretization & $15,\!873.02$ m \\ 
\hline 

$\Delta t$ & Temporal Discretization & 
$50.67$s \\
\hline

\end{tabular}
\end{table}

\end{document}

%% file: tikzFigs/edFig.tikz
\definecolor{evbg0}{HTML}{2d353b}
\definecolor{evbg1}{HTML}{343f44}
\definecolor{evbg2}{HTML}{3d484d}
\definecolor{evbg3}{HTML}{475258}
\definecolor{evbg4}{HTML}{4f585e}
\definecolor{evbg5}{HTML}{56635f}
\definecolor{evfg}{HTML}{d3c6aa}
\definecolor{evred}{HTML}{e67e80}
\definecolor{evorange}{HTML}{e69875}
\definecolor{evyellow}{HTML}{dbbc7f}
\definecolor{evgreen}{HTML}{a7c080}
\definecolor{evaqua}{HTML}{83c092}
\definecolor{evblue}{HTML}{7fbbb3}
\definecolor{evpurple}{HTML}{d699b6}

\begin{tikzpicture}[scale=1.1]

    \fill[evgreen!60] (0,0.625) rectangle (0.8,4.125);
    \node[black, font=\Large\bfseries] at (0.4,2.375) {x};
    
    \fill[evblue!60] 
        (1,0.625) -- 
        (3,1.775) -- 
        (3,2.875) -- 
        (1,4.125) -- 
        cycle;
    
    \node[black, font=\Large] at (2,2.525) {encoder};
    \node[black, font=\large\bfseries] at (2,2.075) {$\mathbf{E}$};
    
    \fill[evyellow!70] (3.2,1.775) rectangle (4,2.875);
    \node[black, font=\Large\bfseries] at (3.6,2.325) {z};
    
    \fill[evpurple!60]
        (4.2,1.775) --
        (6.2,1.0) --
        (6.2,3.55) --
        (4.2,2.875) --
        cycle;
    
    \node[black, font=\Large] at (5.2,2.525) {decoder};
    \node[black, font=\large\bfseries] at (5.2,2.075) {$\mathbf{D}, \mathbf{b}$};
    
    \fill[evaqua!60] (6.4,1.0) rectangle (7.2,3.55);
    \node[black, font=\Large\bfseries] at (6.8,2.275) {y};

    \node[font=\Large, anchor=north] at (3.6,0.3) 
        {$\underbrace{\hspace{5.2cm}}_{\mathbf{A}(\cdot)+\mathbf{b}}$};

\end{tikzpicture}